\documentclass[twoside,11pt]{article}

\usepackage{amsthm}

\usepackage{jmlr2e}

\AtEveryBibitem{\clearlist{language}}

\usepackage[T1]{fontenc}    
\usepackage{amsfonts}       
\usepackage{nicefrac}       
\usepackage{microtype}      
\usepackage[english]{babel}
\usepackage{graphicx}
\usepackage{enumerate} 
\usepackage{multicol}
\usepackage{bm}
\usepackage{amsmath}
\usepackage{amssymb}
\usepackage[colorinlistoftodos]{todonotes}
\usepackage{lmodern} 
\usepackage{comment}
\usepackage{dsfont}
\usepackage{multirow}
\usepackage{hhline}
\usepackage{bm}
\usepackage{listings}
\usepackage{tikz}
\usepackage{booktabs,color,epsfig,url}
\usepackage{wasysym,scalerel,lmodern}
\usepackage{caption}

\numberwithin{equation}{section}

\renewcommand{\P}{{\mathcal P}}

\definecolor{ao(english)}{rgb}{0.0, 0.5, 0.0}

\usetikzlibrary{positioning}
\usetikzlibrary{shapes,snakes}
\usetikzlibrary{shapes,arrows}

\makeatletter

\newcommand{\Rd}{\mathbb{R}^d}

\newcommand{\bL}{\bold{L}}

\newcommand{\beq}{\begin{eqnarray*}}
\newcommand{\eeq}{\end{eqnarray*}}
\newcommand{\beqm}{\begin{eqnarray}}
\newcommand{\eeqm}{\end{eqnarray}}
\DeclareMathOperator*{\esssup}{ess\,sup}

\newcommand{\R}{\mathbb{R}}

\newcommand{\N}{\mathbb{N}}

\newcommand{\cf}{\emph{cf.}~}

\newcommand\tinyprod{\scaleto{\prod}{1.75\LMex}}

\usepackage{stackengine,scalerel}

\newcommand\ocirc[1]{\ThisStyle{\ensurestackMath{%
  \stackon[1pt]{\SavedStyle#1}{\SavedStyle\kern.6\LMpt\circ}}}}

\begin{document}

\title{Uniform Approximation with Quadratic Neural Networks}
\author{\name Ahmed Abdeljawad \email ahmed.abdeljawad@oeaw.ac.at \\
	\addr Johann Radon Institute of Computational and Applied Mathematics (RICAM)\\
	\addr Austrian Academy of Sciences\\
	\addr Altenberger Straße 69, A-4040 Linz, Austria
}

\maketitle

\begin{abstract}
	In this work, we examine the approximation capabilities of deep neural networks utilizing the Rectified Quadratic Unit (ReQU) activation function, defined as \(\max(0,x)^2\), for approximating Hölder-regular functions with respect to the uniform norm. We constructively prove that deep neural networks with ReQU activation can approximate any function within the \(R\)-ball of \(r\)-Hölder-regular functions (\(\mathcal{H}^{r, R}([-1,1]^d)\)) up to any accuracy \(\epsilon \) with at most \(\mathcal{O}\left(\epsilon^{-d /2r}\right)\) neurons and fixed number of layers. This result highlights that the effectiveness of the approximation depends significantly on the smoothness of the target function and the characteristics of the ReQU activation function. Our proof is based on approximating local Taylor expansions with deep ReQU neural networks, demonstrating their ability to capture the behavior of Hölder-regular functions effectively. Furthermore, the results can be straightforwardly generalized to any Rectified Power Unit (RePU) activation function of the form \(\max(0,x)^p\) for \(p \geq 2\), indicating the broader applicability of our findings within this family of activations.
\end{abstract}


\section{Introduction}

Recently, there has been increasing interest in high-dimensional computational problems, which are typically addressed by algorithms that use a finite number of information operations. The complexity of such problems is defined as the minimum number of information operations required to find an approximating solution within an error \(\epsilon\). Remarkable success has been achieved using artificial neural networks for these problems, making this a very active area of research.

Many authors utilize artificial neural networks for various purposes. In particular, for function approximation, under certain conditions, single-hidden-layer neural networks, known as \emph{shallow neural networks}, can effectively approximate continuous functions on bounded domains (see e.g., \cite{Cybenko89Approximationsuperpositionssigmoidala, Hornik91Approximationcapabilitiesmultilayera}). Networks with many hidden layers, referred to as \emph{deep neural networks}, have revolutionized the field of approximation theory, as demonstrated in works such as \cite{Abdeljawad22Approximationsdeepneurala, Blcskei2017OptimalAW, Cheridito19Efficientapproximationhighdimensional, Daubechies22NonlinearApproximationDeepa, DeVore20Neuralnetworkapproximation, Hornung20Spacetimedeepneural, Grohs19Spacetimeerrorestimates, Langer20Approximatingsmoothfunctions, Petersen17Optimalapproximationpiecewise, Schwab21Deeplearninghigh, Voigtlaender19Approximationdeep, Yarotsky2016ErrorBF}. Deep neural networks are also employed in solving partial differential equations \cite{Beck18Solvingstochasticdifferential, Elbrachter18DNNexpressionrate, Geist20Numericalsolutionparametric, Han17Solvinghighdimensionalpartial, Opschoor20DeepReLUnetworksa, E18DeepRitzMethodb}.
In the literature, there are several results concerning the approximation properties of deep neural networks, where different activation functions are explored to unravel the extreme efficiency of these networks to cite some see for instance \cite{Cybenko89Approximationsuperpositionssigmoidala,Ding18Activationfunctionstheir, Glorot10Understandingdifficultytraining, Hornik91Approximationcapabilitiesmultilayera, Ramachandran18Searchingactivationfunctions,Yarotsky2016ErrorBF}.
Specifically, neural networks with \emph{Rectified Linear Units} (ReLU), defined by \(x \mapsto \max(0, x)\), are widely studied both in theoretical approximation and practical applications.

Several theoretical analyses based on ReQU neural networks have been proposed by various authors for different tasks. For instance, \cite{Chui16Deepnetslocal} introduced a ReQU deep learning approach for approximating functions on low-dimensional manifolds, specifically opting to use the ReQU activation function in their construction. Additionally, Mhaskar in \cite{Mhaskar93Approximationpropertiesmultilayered} demonstrated that deep networks with the so-called $k$th-order sigmoidal activation functions can efficiently approximate any continuous function on any compact subset of a Euclidean space.
Mainly, if $k\geq 0$ is an integer,
a $k$th-order sigmoidal activation is
a function $ \sigma :\mathbb{R} \rightarrow \mathbb{R}$
such that 
\begin{align*}
\lim _{x \rightarrow-\infty} \frac{\sigma(x)}{x^k}=0, \quad \lim _{x \rightarrow \infty} \frac{\sigma(x)}{x^k}=1,
\intertext{and for some constant $C>1$,}
|\sigma(x)| \leq C(1+|x|)^k, \quad x \in \mathbb{R}.
\end{align*}
The study of approximation properties of neural networks with $k$th-order sigmoidal activation functions dates back to the early 1990s (see e.g., \cite{Mhaskar93Approximationpropertiesmultilayered, Mhaskar93Neuralnetworkslocalized}).
Notably, the ReQU activation function is a sigmoidal activation of order 2,
additional information and a comprehensive overview of related work can be found in Section \ref{sec:related_work}. 
Recent research by Google Brain team has introduced more efficient alternatives to Transformers, which have been extensively used in numerous advancements in natural language processing. However, the training and inference costs associated with these models have escalated rapidly.
In response, "Primer" (see \cite{So21PrimerSearchingefficient}) offers a solution by reducing the training costs of large Transformer language models. The improvements in Primer can be largely attributed to two simple modifications, one of which involves utilizing squared ReLU activations, that is, ReQU in our terminology.
See also \cite{Hua22Transformerqualitylineara} for Gated Attention Unit.
Motivated by this empirical finding and the fact that ReQU possesses many interesting theoretical properties, we explore the approximation capabilities of ReQU neural networks when dealing with smooth functions.
It is worth to mention that
our approach and objectives differ significantly from those proposed in \cite{Li19Betterapproximationshigh}. Specifically, our focus is on approximating any Hölder function using ReQU networks, where the approximation error measured with respect to the uniform norm. In contrast, authors in \cite{Li19Betterapproximationshigh} aimed to use RePU networks to construct exact representations of polynomials with optimal size. They also provided approximations using RePU networks for Jacobi-weighted Sobolev and Korobov functions. Our techniques and settings are distinct, particularly when providing theoretical upper bounds for smooth function approximations using ReQU deep neural networks.
Additionally, in \cite{Tang19ChebNetEfficientstable}, the authors extended the results from \cite{Li19Betterapproximationshigh} by constructing deep RePU neural networks based on Chebyshev polynomial approximations. Their work also focused on approximations of polynomials in various settings using these networks.
In \cite{Opschoor21ExponentialReLUDNN}, complexity and approximation upper bounds for holomorphic maps in high dimensions within the framework of RePU neural networks were derived. It was observed that the error decay rate in terms of network size for RePU networks is slightly faster compared to what was achieved with ReLU approximations.
Furthermore, ReQU neural networks have been adapted for solving partial differential equations (PDEs) in \cite{Duan21Convergencerateanalysis}. In this study, the authors established a non-asymptotic convergence rate in the $H^1$ norm for the Deep Ritz Method using ReQU networks. They also succeeded in deriving the approximation error of ReQU networks for functions in $H^2$ with respect to the $H^1$ norm. ReQU neural networks demonstrate superior approximation capabilities for smooth functions because they can represent any monomials and the product operator with fewer nodes and without introducing error.

In this effort, we extend recent advances in the approximation theory of deep neural networks to a different setting. Specifically, this paper addresses the approximation of real-valued functions \( f \) with \emph{Hölder smoothness} with \(d\)-dimensional data. We derive an error bound for approximating Hölder smooth functions using neural networks with ReQU activation function. 

The choice of the ReQU activation function is motivated by its ability to represent both the identity and the product operations without error. Moreover, a ReQU neural network becomes smoother as its depth increases. Additionally, the ReQU network's capability to represent any monomial on a bounded domain makes it an interesting activation function from an approximation-theoretical perspective. Our approach builds on well-established techniques for studying the approximation capabilities of deep neural networks for smooth functions, as discussed in \cite{Kohler21rateconvergencefully, Lu20Deepnetworkapproximation}.

To achieve this, we employ a feedforward neural network and investigate how the choice of the ReQU activation function affects the approximation error and network complexity. Specifically, we focus on the approximation rate between the constructed network and a smooth function, with the error measured with respect to the uniform norm. Theorem \ref{thm_main} is the central result of this paper.

An interesting question for future research is how well ReQU networks approximate Hölder smooth functions, or even broader classes of functions, with respect to different error norms, such as Sobolev or Besov norms. Furthermore, exploring different architectures, such as \emph{ResNet} or \emph{convolutional neural networks}, could provide additional insights, as this paper only considers feedforward neural networks. We believe that the use of the ReQU activation function in deep neural networks will yield further valuable insights.

\subsection{Prior and Related Work}\label{sec:related_work}
Here we provide an overview of the most relevant prior work.
Universal approximation theorems (see \cite{Cybenko89Approximationsuperpositionssigmoidala, Hornik91Approximationcapabilitiesmultilayera}) are among the most well-known results in the field of function approximation by neural networks. These theorems assert that any measurable function defined on a bounded domain can be approximated to any desired accuracy using a neural network with a single hidden layer, commonly referred to as a shallow neural network.
Research on the approximation capabilities of shallow neural networks has become increasingly nuanced, exploring various aspects such as the impact of the number of hidden neurons, the properties of activation functions, and the types of functions that can be effectively approximated. This has led to a deeper understanding of the limitations and potential of these networks.
Without aiming for exhaustiveness, key contributions include studies on
the approximation properties of shallow neural networks with RePU activation function were studied in \cite{Klusowski16ApproximationcombinationsReLU} and \cite{Siegel20Highorderapproximationrates}.
The approximation error bounds related to the number of neurons for functions with bounded first moments has been showed by Barron in \cite{Barron93Universalapproximationboundsa, Barron04Approximationestimationbounds}
and extended by many others.
A universal approximation theorem for complex-valued neural networks were
analyzed in \cite{Voigtlaender20universalapproximationtheorem}.
Other authors have investigated the optimal approximation rates for smooth or analytic functions \cite{Mhaskar96Neuralnetworksoptimal}, as well as the lower bound rates, demonstrating that shallow neural networks fail to overcome the curse of dimensionality for certain classes of target functions \cite{Grohs21Lowerboundsartificiala}.
The success of shallow neural networks in approximating a broad range of functions
(e.g., \cite{Abdeljawad22Integralrepresentationsshallowa,Abdeljawad24WeightedSobolevApproximation, Abdeljawad23SpaceTimeApproximationShallowa, Bach17Breakingcursedimensionalitya, Bietti22Learningsingleindexmodels, Chen20dynamicalcentrallimit, DeVore23Weightedvariationspaces, E21Kolmogorovwidthdecaya, Domingo-Enrich21Dualtrainingenergybased, Mhaskar19Dimensionindependentbounds,Wojtowytsch20CanShallowNeuralb}) has naturally led researchers to explore the potential of deeper networks. By increasing the number of layers, different authors sought to enhance the representational power and approximation capabilities of neural networks, potentially overcoming some of the limitations observed in shallow architectures, particularly in handling complex, high-dimensional functions. This shift in focus has paved the way for the development and analysis of deep learning models, which have since become a cornerstone of modern machine learning.

Deep neural networks excel in function approximation by leveraging multiple layers to capture complex patterns, offering superior accuracy and overcoming limitations like the curse of dimensionality.
The universality of deep neural networks with sigmoidal activation functions was established in \cite{Mhaskar93Neuralnetworkslocalized, Chui94Neuralnetworkslocalized}. In contrast, when the activation function is the ReLU, explicit approximation rates were obtained in
 \cite{Blcskei2017OptimalAW, Schmidt-Hieber17Nonparametricregressionusing,Petersen17Optimalapproximationpiecewise, Yarotsky2016ErrorBF} for deep neural networks.
Recent studies have explored the approximation rates of deep neural networks with RePU activation functions across various function spaces. The motivation for studying this family of activation functions lies in their efficiency in representing multiplications and polynomials with fewer layers and neurons compared to ReLU networks.
 To derive approximation results for RePU networks, researchers typically convert splines or polynomials into RePU networks, building on established approximation results for these mathematical constructs.
For instance, in \cite{Achour24generalapproximationlower}, a lower bound on the approximation error for functions belonging to the unit ball of Hölder space has been proved for networks activated by piecewise-polynomial activation functions in the  $L^p(\mu)$ norm, where $\mu$ is a probability measure. Instead in \cite{Opschoor21ExponentialReLUDNN}, exponential error bounds in terms of the total number of weights, with respect to the Sobolev norm, have been investigated for approximating holomorphic functions using RePU networks.
Similarly, in \cite{Abdeljawad22Approximationsdeepneurala, Duan21Convergencerateanalysis,Li19PowerNetEfficientrepresentations},  approximation error rates have been provided in terms of the total number of neurons (and non-zero weights) when approximating Sobolev spaces, measured in the Sobolev norm.
The authors in \cite{Belomestny22Simultaneousapproximationsmooth} investigated approximation rates for functions in Hölder spaces with respect to the Hölder norm. In contrast to our results, we do not restrict the Hölder regularity to be in $(2,\infty)$. Moreover, the technique used in our paper is entirely different from that in \cite{Belomestny22Simultaneousapproximationsmooth}, as we do not represent tensor-product splines using ReQU networks in our approach.
In the following table, we compare our rates with known theoretical limits in similar approximation settings, focusing exclusively on the ReQU ($\max(0,x)^2$)and ReCU  ($\max(0,x)^3$) activation functions.

\begin{table}[!h]
	\begin{center}
\begin{tabular}{|c|c|c|c|c|c|c|}
	\hline
	& Error Norm
	&
	\begin{tabular}{l} 	Activation\end{tabular}
	&
	\begin{tabular}{l} Parameters \end{tabular}
	 \\
	\hline Abdeljawad and Grohs \cite{Abdeljawad22Approximationsdeepneurala} & $W^{s, p}$ & ReCU & $\mathcal{O}\left(\epsilon^{-d /(r-s)}\right)$\\
	\hline Belomestny et al. \cite{Belomestny22Simultaneousapproximationsmooth} & $\mathcal{H}^s$ & ReQU & $\mathcal{O}\left(\epsilon^{-d /(r-s)}\right)$ \\
	\hline Li et al. \cite{Li19Betterapproximationshigh, Li19PowerNetEfficientrepresentations} & $L^2$ & ReQU & $\mathcal{O}\left(\epsilon^{-d / r}\right)$ \\
	\hline	This work & $L^\infty$ & ReQU & $\mathcal{O}\left(\epsilon^{-d /2r}\right)$ \\
	\hline
\end{tabular}
	\caption{Approximation results of ReQU and ReCU neural networks on a function with smoothness order $r>0$, within the accuracy level $\epsilon$.}
\end{center}
\end{table}

\subsection{Contributions}

The main contributions of our work can be summarized as follows:

\begin{itemize}
	\item We investigate the expressivity and approximation properties of deep neural networks activated by the ReQU activation function. This enables us to establish the necessary depth and width of the network to approximate a smooth function with a desired level of accuracy. Specifically, we demonstrate a worst-case approximation rate for functions from H\"older spaces.
	\item We show that the depth of the network is independent of the accuracy level, depending only on the regularity of the function spaces and  on the dimension of the data both logarithmically (see Theorem \ref{thm_main}). Furthermore, a univariate case is derived as a consequence of the main theorem, where we also establish a lower bound on the accuracy in terms of the number of neurons and the regularity of the function.
\end{itemize}
	It is worth mentioning that all of our proofs are constructive, meaning they explicitly demonstrate how to build deep neural networks activated by the ReQU function to achieve the desired convergence rates. Many interesting questions remain open regarding the constructive approximation of smooth functions with deep neural networks, such as the approximation of analytic functions, Gevrey-regular functions, and generalized functions. We believe that establishing a lower bound on the rate is crucial to understanding the approximation capabilities of neural networks. Furthermore, ReQU networks are valuable in exploring new loss functions, particularly norms that involve derivatives to measure error. In our paper, we limit ourselves to the uniform error, but more general norms could be investigated when dealing with smooth functions using ReQU networks, which we have postponed for future work.

\subsection{Notation}

We use the following notations in our article:
For a $d$ -dimensional multiple index $\alpha \equiv\left(\alpha_{1},
\ldots, \alpha_{d}\right) \in \mathbb{N}_{0}^{d}$
where $\mathbb{N}_{0}:=\mathbb{N} \cup\{0\}$.
We denote by \(\lfloor\cdot \rfloor \) the floor function,
moreover $\|  \alpha\|_{\ell^0}$ denotes the number of non zero elements in
the multi-index $\alpha$.
We let $|\alpha|= \sum_{i=1}^d \alpha_i$ and 
$x^{\alpha}:=x_{1}^{\alpha_{1}} \cdots x_{d}^{\alpha_{d}}$
where $x \in \mathbb{R}^{d}$.
For a function
$f: \Omega \rightarrow \mathbb{R}$,
where $ \Omega$ denotes the domain of the function,
we let $\|f\|_{\infty}:=\sup _{x \in \Omega}|f({x})|$.
We use notation
$$
	D^\alpha f:=\frac{\partial^{|\mathrm{\alpha}|} f}{\partial {x}^{\alpha}}=
	\frac{\partial^{|\alpha|} f}{\partial x_{1}^{\alpha_{1}}
	\cdots \partial x_{d}^{\alpha_{d}}}
$$
for $\alpha \in \mathbb{N}_{0}^{d}$ to denote the derivative of $f$ of order $\alpha$.
We denote by $\mathcal{C}^{m}( \Omega)$, the space of $m$ times differentiable functions
on $ \Omega$ whose partial derivatives of order $\alpha$ with $|\alpha| \leq m$ are continuous.

If $C$ is a cube we denote the "bottom left" corner of $C$
by $\bold{C^L}$, Figure \ref{fig:blc} shows $\bold{C^L}$ in case $d=2$
for the square $[-1,1]^2$.
\begin{figure}[h]
	\centering
	\begin{tikzpicture}
		\coordinate (1) at (-1,-1);
		\coordinate (2) at (1,-1);
		\coordinate (3) at (1,1);
		\coordinate (4) at (-1,1);
		\coordinate (5) at (-1,-1);
		\fill (1) circle (2pt) node [below] {$\bold{C^L}$};
		\draw[help lines, color=gray!30, dashed] (-2.3,-2.3) grid (2.3,2.3);
		\draw[->,ultra thick] (-3,0)--(3,0) node[right]{};
		\draw[->,ultra thick] (0,-3)--(0,3) node[above]{};
		\draw (1)--(2)--(3)--(4)--(5);
	\end{tikzpicture}
	\caption{ $\bold{C^L}$ is the bottom left corner of the square $[-1, 1]^2$.} \label{fig:blc}
\end{figure}

\par

Therefore, each half-open cube $C$
with side length $s$
can be written as a polytope defined by
\begin{align*}
C= \{ x\in \mathbb{R}^d: \;-x_{j} +\bold{C^L}_{j} \leq 0 \ \mbox{and} \ x_{j} -
\bold{C^L}_{j}-s < 0 \quad (j \in \{1, \dots, d\})\}.
\end{align*}
Furthermore, we describe by $\ocirc{C}_{\delta} \subset C$ the cube, which contains all ${x} \in C$
that lie with a distance of at least $\delta$ to the boundaries of $C$, i.e. a polytope defined by
\begin{align*}
	\ocirc{C}_{\delta}  = \{ x\in \mathbb{R}^d: \; 
	-x_{j} + \bold{C^L}_{j} \leq - \delta\
	\mbox{and} \ x_{j} - \bold{C^L}_{j}-s < -\delta \quad (j \in \{1, \dots, d\})\}.
\end{align*}
If $\P$ is a partition of cubes of $[-1,1)^d$
and ${x} \in [-1,1)^d$, then we denote the cube $C \in \P$,
which satisfies ${x} \in C$, by $C_\P ({x})$.

\subsection{Outline}

The paper is organized as follows. In Section \ref{sec:prelim}, we briefly describe the class of functions used in our study and introduce the relevant definitions of neural networks. Section \ref{se3} presents a detailed analysis of the approximation error and complexity for Hölder regular functions using feedforward deep neural networks with the ReQU activation function.

\section{Preliminaries}\label{sec:prelim}

\subsection{Functions of H\"older smoothness}

The paper revolves around what we informally describe as ``functions of smoothness $r$'' for any $r>0$. 
Let $\Omega\subseteq\Rd$, if $r$ is integer,
 we consider the standard Sobolev space $\mathcal{W}^{r, \infty}(\Omega)$ with the norm
\begin{align*}
    \|f\|_{\mathcal{W}^{r, \infty}(\Omega)} =
    \max_{|\alpha| \leq r} \esssup_{x \in\Omega} |D^{\alpha} f(x)|.
\end{align*}
Here $D^{\alpha}f$ denotes the (weak) partial derivative of $f$.
For \(f \in \mathcal{W}^{r, \infty}(\Omega)\), the derivatives \(D^{\alpha}f\) of order \(|\alpha| < r\) exist in the strong sense and are continuous. The derivatives \(D^{\alpha}f\) of order \(|\alpha| = r - 1\) are Lipschitz, and \(\max_{\alpha: |\alpha| = r} \operatorname{esssup}_{x \in \Omega} |D^{\alpha} f(x)|\) can be upper- and lower-bounded in terms of the Lipschitz constants of these derivatives.

In the case of non-integer \(r\), we consider Hölder spaces that provide a natural interpolation between the above Sobolev spaces. For any non-negative real number \(r\), we define the Hölder space \(\mathcal{H}^{r}(\Omega)\) as a subspace of \(\lfloor r \rfloor\) times continuously differentiable functions having a finite norm
\begin{align*}
	\|f\|_{\mathcal{H}^{r}(\Omega)} =
	\max \Big\{ \|f\|_{\mathcal{W}^{ \lfloor r \rfloor, \infty}(\Omega)},
	\max_{|\alpha| = \lfloor r \rfloor}
	\sup_{\substack{ x,  y \in \Omega  \\  x \neq y}}
	\dfrac{|D^{\alpha} f(x) - D^{\alpha} f(y)|}{\| x - y \|^{r - \lfloor r \rfloor}} \Big \}.
\end{align*}

We denote by \(\mathcal{H}^{r, R}(\Omega)\) the closed ball in the Hölder space of radius \(R\) with respect to the Hölder norm, i.e.,
$$
\mathcal{H}^{r, R}(\Omega) := \left\{f \in \mathcal{H}^{r}(\Omega)
: \|f\|_{\mathcal{H}^{r}(\Omega)} \leq R \right\}.
$$
Given a non-integer \(r\), we define ``$r$-smooth functions'' as those belonging to \(\mathcal{C}^{\lfloor r \rfloor, r - \lfloor r \rfloor}(\Omega)\), where \(\lfloor \cdot \rfloor\) is the floor function.

\subsection{Mathematical definitions of neural networks}

In this section, we provide a brief introduction to deep neural networks from a functional analytical perspective. 
We will cover some fundamental properties of these networks, such as concatenation and parallelization. Note that in this paper, we focus on neural networks with a fixed \emph{architecture}.
A widely recognized architecture is the feedforward neural network architecture, which represents a function as a sequence of affine-linear transformations followed by a componentwise application of a non-linear function, known as the \emph{activation function}. We will start by defining the concept of an architecture.

\par

\begin{definition}\label{def:architecture}
    Let $d, L \in \mathbb{N}$, a neural network architecture $\mathcal{A}$
    with input dimension $d$ and $L$ layers is a sequence of matrix-vector tuples
    $$
        \mathcal{A}=\left(\left(A_{1}, b_{1}\right),\left(A_{2}, b_{2}\right),
        \ldots,\left(A_{L}, b_{L}\right)\right)
    $$
    such that $N_{0}=d$ and $N_{1}, \ldots, N_{L} \in \mathbb{N}$,
    where each $A_{l}$ is an $N_{l} \times \sum_{k=0}^{l-1} N_{k}$ matrix,
    and $b_{l}$ is a vector of length $N_{l}$ with elements in $\{0,1\}$. 
    A neural network architecture is essentially a neural network characterized by binary weights.
\end{definition}

\par

Once the architecture of a neural network is fixed, we proceed to define its specific realization, which involves specifying the activation functions used within the network. The realization of the network encompasses the detailed implementation of these functions, determining how they transform the input data through each layer.

\begin{definition}
Let $d, L \in \mathbb{N}$, $\rho: \mathbb{R} \rightarrow \mathbb{R}$ is arbitrary function
and let  $\mathcal{A}$ be an architecture
defined as follows:
$$
    \mathcal{A}=\left(\left(A_{1}, b_{1}\right),\left(A_{2}, b_{2}\right),
    \ldots,\left(A_{L}, b_{L}\right)\right)
$$
where $N_{0}=d$ and $N_{1}, \ldots, N_{L} \in \mathbb{N}$,
and where each $A_{\ell}$ is an $N_{\ell} \times N_{\ell-1}$ matrix,
and $b_{\ell} \in \mathbb{R}^{N_{\ell}}$.
Then we define the neural network $\Phi$ with input dimension $d$ and $L$ layers
as the associated realization of $\mathcal{A}$
with respect to the activation function $\rho$ as the map 
$\Phi:=\mathrm{R}_{\rho}(\mathcal{A}):
	\mathbb{R}^{d} \rightarrow \mathbb{R}^{N_{L}}$
such that
$$
	\Phi:= {R}_{\rho}(\Phi)(x)=x_{L}
$$
where $x_{L}$ results from the following scheme:
\begin{align*}
x_{0}:= & x \\
x_{\ell}:= &\rho\left(A_{\ell} x_{\ell-1}+b_{\ell}\right),
\quad \text{ for }\ell=1, \ldots, L-1 \\
x_{L}:= & A_{L} x_{L-1}+b_{L}
\end{align*}
and $\rho$ acts componentwise, i.e., for a given vector $y\in \mathbb{R}^{m}$,
$\rho(y)=\left[\rho\left(y_{1}\right), \ldots, \rho\left(y_{m}\right)\right]$.
\end{definition}

\par

We call $N(\Phi):=\max(d, N_{1}, \dots, N_L)$ the maximum number of neurons
per layer of  the number of the network $\Phi$,
while $L(\Phi):=L-1$ denotes the number
of hidden layers of $\Phi $,
hence we write $\Phi \in \mathtt{N}_\varrho(L(\Phi),N(\Phi))$.
Moreover,
$M(\Phi):=\sum_{j=1}^{L}\left(\left\|A_{j}\right\|_{\ell^{0}}+
    \left\|b_{j}\right\|_{\ell^{0}}\right)$
denotes the total
number of nonzero entries of all $A_{\ell}, b_{\ell},$ which we call
the number of weights of $\Phi$.
Moreover, $N_{L}$  denotes the dimension of the output layer of $\Phi$.

\par

We recall that throughout the paper we consider the Rectified Quadratic Unit (ReQU)
activation function, which is defined as follows:
\begin{equation*}\label{eq:recu}
\rho_2: \mathbb{R} \rightarrow \mathbb{R}, \quad x \mapsto \max (0, x)^2.
\end{equation*}

\begin{remark}
	Our theory works
	with any Rectified Power Unit activation 
	(RePU) $\max(0,x)^p$ of order $p\geq 2$.
	Since 
	$$
	 \max(0,x)^p= x^{p-2}\cdot \max(0, x)^2
	\text{ for any }
	p\geq 2,\; x\in \mathbb{R}
	$$
	and there exists a ReQU network which represents the product
	furthermore any monomial can be represented exactly with a ReQU
	network.
\end{remark}
To construct new neural networks from existing ones, we will frequently need to concatenate
networks or put them in parallel, cf., \cite{Petersen17Optimalapproximationpiecewise} for more details.
We first define the concatenation of networks.

\begin{definition}\label{def:concat}
Let $L_{1}, L_{2} \in \mathbb{N}$, and 
let $\Phi^{1}$ and $\Phi^{2}$ be two neural networks
where the input layer of $\Phi^{1}$ has the same dimension as the output layer of  $\Phi^{2}$,
where
$$
	\mathcal{A}^{1}=\left(\left(A_{1}^{1}, b_{1}^{1}\right),
	\ldots,\left(A_{l_{1}}^{1}, b_{l_{1}}^{1}\right)\right),
	\quad \mathcal{A}^{2}=\left(\left(A_{1}^{2}, b_{1}^{2}\right),	
	\ldots,\left(A_{l_{2}}^{2}, b_{l_{2}}^{2}\right)\right)
$$
 are their respective architectures.
such that the input layer of $\mathcal{A}^{1}$ has the same dimension
as the output layer of $\mathcal{A}^{2}$.
Then, $\mathcal{A}^{1} \bullet \mathcal{A}^{2}$ denotes
the following $L_{1}+L_{2}-1$ layer architecture:
\[
	\begin{aligned}
	    \mathcal{A}^{1} \bullet \mathcal{A}^{2}:=&\left(\left(A_{1}^{2}, b_{1}^{2}\right),
	    \ldots,\left(A_{L_{2}-1}^{2}, b_{L_{2}-1}^{2}\right),\left(A_{1}^{1}
	    A_{L_{2}}^{2}, A_{1}^{1} b_{L_{2}}^{2}+b_{1}^{1}\right),
	    \left(A_{2}^{1}, b_{2}^{1}\right), \ldots,\left(A_{L_{1}}^{1}, b_{L_{1}}^{1}\right)\right).
	\end{aligned}
\]
We call $\mathcal{A}^{1} \bullet \mathcal{A}^{2}$ the concatenation
of $\mathcal{A}^{1}$ and $\mathcal{A}^{2}$,
moreover
$\Phi^1(\Phi^2):={R}_{\rho_2}\left(\mathcal{A}^{1} \bullet \mathcal{A}^{2}\right)$
is the realization of the concatenated networks.
\end{definition}

\par

Besides concatenation, we need another operation between networks,
that is the \emph{parallelization},
where we can put two networks of same length in parallel.

\par

\begin{definition}\label{def:parallel_net}
Let $L \in \mathbb{N}$ and let ${\Phi}^{1}, \Phi^{2}$ be two neural networks
with $L$ layers and $d$-dimensional input,
where  $\mathcal{A}^{1}=\left(\left(A_{1}^{1}, b_{1}^{1}\right),
\ldots,\left(A_{L}^{1}, b_{L}^{1}\right)\right)$
and $\mathcal{A}^{2}
=\left(\left(A_{1}^{2}, b_{1}^{2}\right), \ldots,\left(A_{L}^{2}, b_{L}^{2}\right)\right)$ be
their architectures respectively.
We define
$$
    {P}\left({\Phi}^{1}, \Phi^{2}\right):=\left(\left(\widetilde{A}_{1}, \widetilde{b}_{1}\right),
    \ldots,\left(\tilde{A}_{L}, \widetilde{b}_{L}\right)\right)
$$
where
$$
	\tilde{A}_{1}:=\left(\begin{array}{c}
	A_{1}^{1} \\
	A_{1}^{2}
	\end{array}\right), \quad \tilde{b}_{1}:=\left(\begin{array}{c}
	b_{1}^{1} \\
	b_{1}^{2}
	\end{array}\right) \quad \text { and } \quad \tilde{A}_{\ell}:=\left(\begin{array}{cc}
	A_{\ell}^{1} & 0 \\
	0 & A_{\ell}^{2}
	\end{array}\right), \quad \tilde{b}_{\ell}:=\left(\begin{array}{c}
	b_{\ell}^{1} \\
	b_{\ell}^{2}
	\end{array}\right) \quad \text { for } 1<\ell \leq L.
$$
Then, ${P}\left(\Phi^{1}, \Phi^{2}\right)$ is a neural network with $d$-dimensional
input and $L$ layers, called the parallelization of $\Phi^{1}$ and $\Phi^{2}$.
\end{definition}

\section{Approximation error of smooth functions by
deep ReQU neural network}
\label{se3}

The aim of this section is to present a new result regarding the approximation of \(r\)-smooth functions within a ball of radius \(R\)  (i.e., $ \mathcal{H}^{r, R}(\mathbb{R}^d)$) using deep neural networks with ReQU activation. A key feature of the ReQU function is its ability to represent both the identity and multiplication operations exactly, without any error. Additionally, ReQU networks exhibit smoothness, and this smoothness improves as the network becomes deeper. In contrast, neural networks with ReLU activation functions can only approximate multiplication with a certain degree of error, which affects the overall approximation quality. 

\begin{theorem}\label{thm_main}
 Let $r, R>0$,  $f\in \mathcal{H}^{r, R}(\mathbb{R}^d)$ and $M \in \N$  such that 
$M>\left(\frac{cRd^{\nicefrac r2}}{\epsilon}\right)^{\nicefrac{1}{2r}}$,
 for any $\epsilon \in (0, 1)$ and $c>0$ in \eqref{eq:taylor_error} .
Then there exists a ReQU neural network $ \Phi_f \in \mathtt{N}_{\rho_2}(L(\Phi_f), N(\Phi_f))$,
satisfies
\begin{align*}
 \| \Phi_f - f\|_{L^\infty( [-1,1]^d)} \leq\epsilon,
  \label{th2eq1}
\end{align*}
where 
\begin{align*}
L(\Phi_f ) &=
		\lfloor\log_2(\lfloor r \rfloor)\rfloor+ 2\lfloor\log_2(d+1+d \lfloor\log_2(\lfloor r \rfloor)\rfloor)\rfloor +8,
	\\
N(\Phi_f) &= 2^d\Big(\max\left((1+\binom{d+\lfloor r \rfloor}{d} ) M^d\max(4, 2d+1)+2,
			\;2{\binom{d+\lfloor r \rfloor}{d}}(d+1+  d \lfloor\log_2(\lfloor r \rfloor)\rfloor)\right)
	\\
		&\qquad\qquad+ 2(M^d(2d+1)+2d + 2dM^d)+ 2+ M^d\max(4, 2d+1)\Big).
\end{align*}
\end{theorem}

\par

The following result is an immediate consequence of Theorem \ref{thm_main}, and its proof is therefore straightforward.
\begin{corollary}[Univariate Case for the Approximation Rate]\label{cor:main}
	Under the same assumptions of the previous theorem, if instead $d=1$, then there exists a ReQU neural network $ \Phi_f \in \mathtt{N}_{\rho_2}(L(\Phi_f), N(\Phi_f))$,
	such that
	\begin{align*}
		\| \Phi_f - f\|_{L^\infty( [-1,1])} \leq\epsilon,
	\end{align*}
	where 
	\begin{align*}
		L(\Phi_f ) &=
		\lfloor\log_2(\lfloor r \rfloor)\rfloor+ 2\lfloor\log_2(2+ \lfloor\log_2(\lfloor r \rfloor)\rfloor)\rfloor +8,
		\\
		N(\Phi_f) &=\begin{cases}
			4(\lfloor r \rfloor+1)(2M +1) &\text{ if }M\geq \frac{\lfloor r \rfloor	\lfloor\log_2(\lfloor r \rfloor)\rfloor +\lfloor\log_2(\lfloor r \rfloor)\rfloor + 2\lfloor r \rfloor + 4}{2\lfloor r \rfloor -5},
			\\
			4
			((\lfloor r \rfloor+1)(2+   \lfloor\log_2(\lfloor r \rfloor)\rfloor)
			+ 7M+3),& \text{otherwise}.
		\end{cases}
	\end{align*}
\end{corollary}

\begin{remark}[Lower Bound on the Error Rate]
	In the univariate case, i.e., when \(d = 1\), using Corollary \ref{cor:main}, it is easy to determine the dependence of the rate \(\epsilon\) on the number of neurons in the network. For instance, if 
	\[
	M \geq \frac{\lfloor r \rfloor \lfloor \log_2(\lfloor r \rfloor) \rfloor + \lfloor \log_2(\lfloor r \rfloor) \rfloor + 2\lfloor r \rfloor + 4}{2\lfloor r \rfloor - 5},
	\]
	and knowing that by assumption 
	\[
	M > \left(\frac{cR }{\epsilon}\right)^{\nicefrac{1}{2r}},
	\]
	it follows that 
	\[
	\epsilon > cR\left(\frac{1}{2}\left(\frac{N(\Phi_f)}{4(\lfloor r \rfloor + 1)}-1\right)\right)^{-2r}.
	\]
	
\end{remark}

The proof of our main result Theorem \ref{thm_main} builds upon the results presented in \cite[Theorem 2(a)]{Kohler21rateconvergencefully}. The subsequent result demonstrates that any \(r\)-smooth function can be approximated by a Taylor polynomial. This result is pivotal to the approximation strategy employed in this work, as it allows us to construct \(r\)-smooth functions using piecewise Taylor polynomials. For further details on the proof of Lemma \ref{lem:r-smooth_taylor_approximation}, we refer the reader to the proof of \cite[Lemma 1]{Kohler14Optimalglobalrates}.

\begin{lemma}\label{lem:r-smooth_taylor_approximation}\label{le1a}
Let $r, R>0$, and $u\in \mathcal{H}^{r, R}(\mathbb{R}^d)$.
Moreover, for any fixed ${x}_0 \in \Rd$, let $T_{{x}_0}^{\lfloor r\rfloor}u$ denotes the Taylor
polynomial of total degree ${\lfloor r\rfloor}$ around ${x}_0$
defined by
\begin{equation*}\label{eq:taylor_expanssion}
T_{{x}_0}^{\lfloor r\rfloor}u({x}) =
\sum_{\substack{\alpha \in \N_0^d:
    |\alpha| \leq \lfloor r\rfloor}}
    D^{\alpha} u ( {x}_0)\cdot
\frac{\left({x} - {x}_0\right)^{\alpha}}{\alpha!}.
\end{equation*}

Then, for any ${x} \in \Rd$
\begin{align*}
\left|u({x}) - T_{{x}_0}^{\lfloor r\rfloor}u({x})\right|
\leq
c \cdot R \cdot \Vert x - x_0 \Vert^r
\end{align*}
holds for a constant $c$ depending  on $\lfloor r\rfloor$ and $d$ only.
\end{lemma}
In the proof of our main theorem, we utilize a piecewise Taylor polynomial as outlined in Lemma \ref{lem:r-smooth_taylor_approximation}. To achieve this, we divide the domain \([-1,1)^d\) into \(M^d\) and \(M^{2d}\) half-open equi-volume cubes of the following forms:
\[
[\bm{a},\bm{b})=[\bm{a}_1,\bm{b}_1) \times \dots \times [\bm{a}_{d_i},\bm{b}_{d_i}),
\quad \bm{a}, \bm{b} \in \mathbb{R}^{d_i},
\]
were $d_1 = d$ and $d_2=2d$ respectively.
We then fix two partitions, $\mathcal{P}_1$ and $\mathcal{P}_2$ of
half-open equivolume cubes defined as follows:
\begin{align}\label{partition}
	\mathcal{P}_1=\{B_{k}\}_{k \in \{1, \dots, M^d\}} 
	\ \mbox{and} \
	\mathcal{P}_2=\{C_{k}\}_{k \in \{1, \dots, M^{2d}\}}.
\end{align}
For each $j \in \{1, \dots, M^d\}$ we denote
the cubes of $\mathcal{P}_2$ contained within $B_{j}$
by ${C}_{1, j}, \dots, {C}_{M^d, j}$.
We order these cubes such that the bottom left corner
$(\bold{C}_{i, j})^{\bold{L}}$ of ${C}_{i, j}$ can be expressed as
\begin{align*}\label{tildec}
(\bold{{C}}_{i, j})^{\bold{L}} = {v}^{(i)} +
(\bold{B}_{j})^{\bold{L}},
\end{align*}
for all $i, j \in \{1, \dots, M^d\}$ and for some vector $ {v}^{(i)} $ with entries in
$\{0, 2/M^2, \dots, (M-1) \cdot 2/M^2\}$.
The vector $ {v}^{(i)} $ describes the position of bottom left corner
$(\bold{{C}}_{i, j})^\bold{L}$ relative to $(\bold{B}_{j})^{\bold{L}}$.
 The cubes are ordered such that this position is independent of \(j\). Consequently, the partition $\P_2$ can be represented by the cubes ${C}_{i, j}$ as follows:
\begin{align*}
\P_2=\{{C}_{i, j}\}_{i \in \{1, \dots, M^d\}, j \in \{1, \dots, M^d\}}.
\end{align*}

Moreover, the Taylor expansion of a function $f\in \mathcal{H}^{r, R}(\mathbb{R}^d)$
given by \eqref{eq:taylor_expanssion}
can be computed by the piecewise Taylor polynomial defined on $\P_2$.
In particular, we have
\[
	T_{(\bold{C}_{\P_2}({x}))^{\bold{L}}}^{\lfloor r \rfloor}f({x})
	=
	\sum_{i, j \in \{1, \dots, M^d\} } T_{(\bold{{C}}_{i, j})^{\bold{L}}}^{\lfloor r \rfloor}f({x})
	\cdot \mathds{1}_{{C}_{i, j}}({x})
\]
then, we have for any $x \in [-1, 1)^d$
\begin{align}\label{eq:taylor_error}
	\begin{split}
		\left|	f({x}) - T_{(\bold{C}_{\P_2}({x}))^{\bold{L}}}^{\lfloor r \rfloor}f({x})\right|
		&=
		  \left|\sum_{i, j \in \{1, \dots, M^d \}}
			\left(f(x) -T_{(\bold{{C}}_{i, j})^{\bold{L}}}^{\lfloor r \rfloor}f({x})\right)
		\cdot \mathds{1}_{{C}_{i, j}}({x})\right|
		\\
		&
		\leq c R \sum_{i, j \in \{1, \dots, M^d\} }\| x -(\bold{{C}}_{i, j})^{\bold{L}}\|^r
		\mathds{1}_{{C}_{i, j}}({x})
		\\
		& \leq c R  \left(\frac{2 \sqrt{d}}{M^{2}}\right)^r.
	\end{split}
\end{align}
To achieve our goal of approximating the function \( f \) using neural networks, we start by introducing a recursive definition for the Taylor polynomial \( T_{(\bold{C}_{\mathcal{P}_2}(x))^{\bold{L}}}^{\lfloor r \rfloor}f(x) \). 
For $x\in [-1, 1)^d$, let $C_{\P_1}(x) = B_{j}$ such that $j \in \{1, \dots, M^d\}$.
To that aim, we begin by computing the value of
$(\bold{C}_{\P_1}({x}))^{\bold{L}}=(\bold{B}_{j})^{\bold{L}}$
and the values of $(\partial^{\alpha}f) ((\bold{{C}}_{i,j})^{\bold{L}})$
for $i \in \{1, \dots, M^d\}$ and $\alpha \in \N_0^d$ with $|\alpha| \leq \lfloor r \rfloor$.
To accomplish this, we need to compute the product of the indicator function with \( (\bold{B}_{j})^{\bold{L}} \) or \( (\partial^{\alpha} f)((\bold{C}_{i, j})^{\bold{L}}) \) for each \( j \in \{1, \dots, M^d\} \), respectively. The value of \( x \) is used in our recursion, so we shift it by applying the identity function.
\begin{flalign}\label{eq:rec_phi}
&&	\phi^{(0)} &= (\phi^{(0)} _{ 1}, \dots, \phi^{(0)} _{d}) = {x},&\notag
\\[-1ex]
&&	\phi^{(1)}& = (\phi^{(1)}_{ 1}, \dots, \phi^{(1)}_{ d}) =
 \sum_{j \in \{1, \dots, M^d\}} (\bold{B}_{j})^{\bold{L}} \cdot \mathds{1}_{B_{j}}({x})&\notag
\\[-1ex]
\rlap{and}\\[-1ex]
&&	\phi^{(\alpha, i)} _f&=
 \sum_{j \in \{1, \dots, M^d\}} (\partial^{\alpha} f)
 \left((\bold{{C}}_{i, j})^{\bL}\right) \cdot \mathds{1}_{B_{j}}({x}),&\notag
\end{flalign}
for $i \in \{1, \dots, M^d\}$ and  $\alpha \in \N_0^d$ such that $| \alpha | \leq \lfloor r \rfloor$.

In a similar way to the previous computation,
we let ${C}_{\P_2}({x})={{C}}_{i, j}$
for any $i, j \in \{1, \dots, M^d\}$. Moreover,
we compute the value of  $(\bold{C}_{\P_2}({x}))^\bL=(\bold{{C}}_{i, j})^\bL$
and the values of $(\partial^\alpha f)\left((\bold{C}_{\P_2}({x}))^\bL\right)$
for any  $\alpha \in \N_0^d$ with $|\alpha| \leq \lfloor r \rfloor$.
We recall that $(\bold{{C}}_{i, j})^{\bold{L}} =  {v}^{(i)} +
(\bold{B}_{j})^{\bold{L}}$, then each cube $C_{i, j}$ can be defined as follows:
\begin{align}\label{Aj}
\mathcal{A}^{(i)} = &\left\{{x} \in \Rd: -x_{k} + \phi^{(1)}_{ k} +  {v}^{(i)} _k
\leq 0 \right. \notag\\
 & \hspace*{1.8cm} \left. \mbox{and } x_{k} - \phi^{(1)}_{ k} - v^{(i)} _k - \frac{2}{M^2} < 0
 \mbox{ for all } k \in \{1, \dots, d\}\right\}.
\end{align}
Therefore, we compute the product of the indicator function $\mathds{1}_{\mathcal{A}^{(i)}}$
by ${\phi}^{(1)} + {v}^{(i)}$ or $\phi^{(\alpha, i)}_f$ for any
$i \in \{1, \dots, M^d\}$,  $\alpha \in \N_0^d$ with $|\alpha| \leq \lfloor r\rfloor$.

Once again we shift the value of $x$ by applying the identity function.
We set
\begin{flalign}\label{eq:rec_psi}
    &&\psi^{(0)} &= (\psi^{(0)}_{1}, \dots, \psi^{(0)}_{d})= {\phi}^{(0)},&\notag
    \\[-1ex]
    &&\psi^{(1)}&= (\psi^{(1)}_{1}, \dots, \psi^{(1)}_{d})=
    \sum_{i=1}^{M^d} ({\phi}^{(1)}+{v}^{(i)}) \cdot \mathds{1}_{\mathcal{A}^{(i)}}
    \left({\phi}^{(0)}\right)&\notag
    \\[-1ex]
    \rlap{and}\\[-1ex]
    &&\psi^{(\alpha)}_f &= \sum_{i=1}^{M^d} \phi^{(\alpha, i)}_f \cdot \mathds{1}_{\mathcal{A}^{(i)}}
    \left({\phi}^{(0)}\right)&\notag
\end{flalign}
for $\alpha \in \N_0^d$ with $|\alpha| \leq \lfloor r \rfloor$.
In a last step we compute the Taylor polynomial by
\begin{equation}\label{eq:psi_def}
    \psi ^{\lfloor r \rfloor}_f= \sum_{\substack{|\alpha| \leq \lfloor r \rfloor}} \frac{\psi^{(\alpha)}_f}{\alpha!}
    \cdot \left(\psi^{(0)} - \psi^{(1)}\right)^{\alpha}.
\end{equation}

The previous recursion computes the  piecewise Taylor polynomial as Lemma \ref{supple3} shows.
The proof of next result can be found in \cite{Kohler21rateconvergencefully}.

\begin{lemma}
\label{supple3}
  Let $r, R>0$,  $x \in [-1,1)^d$ and $f\in \mathcal{H}^{r, R}(\mathbb{R}^d)$
  such that $T_{(\bold{C}_{\mathcal{P}_2}({x}))^{\bL}}^{\lfloor r \rfloor}f({x})$
  is the Taylor polynomial of total degree $\lfloor r \rfloor$ around $(\bold{C}_{\mathcal{P}_2}({x}))^{\bL}$.
  Define $\psi ^{\lfloor r \rfloor}_f$ recursively as \eqref{eq:psi_def}.
  Then we have
  \[
	\psi ^{\lfloor r \rfloor}_f
	=
	T_{(\bold{C}_{\mathcal{P}_2}({x}))^{\bL}}^{\lfloor r \rfloor}f({x}).
  \]
\end{lemma}

The next result demonstrates that  for any
$x \in \bigcup_{k \in \{1, \dots, M^{2d}\}}\ocirc{(C_{k})}_{1/M^{2r+2}}$
we can approximate $r$-smooth functions  using ReQU neural networks.
 In other words, our network serves as an effective approximator for $r$-smooth functions within an equivolume cube, provided that the approximation is performed away from the cube's boundary.

\begin{lemma}\label{le5}
Let $r, R>0$,  $f\in \mathcal{H}^{r, R}(\mathbb{R}^d)$ and $M \in \N$  such that 
$M>\left(\frac{cRd^{\nicefrac r2}}{\epsilon}\right)^{\nicefrac{1}{2r}}$,
 for any $\epsilon \in (0, 1)$ and $c>0$ in \eqref{eq:taylor_error} .
Then for any $x \in \bigcup_{k \in \{1, \dots, M^{2d}\}}\ocirc{(C_{k})}_{1/M^{2r+2}}$,
there exists a ReQU neural network
$\Psi_f^{\lfloor r\rfloor}(x)\in \mathtt{N}_{\rho_2}(L(\Psi_f^{\lfloor r\rfloor}), N(\Psi_f^{\lfloor r\rfloor}))$
where
\begin{align*}
    L(\Psi ^{\lfloor r \rfloor}_f ) &=\lfloor\log_2(\lfloor r \rfloor)\rfloor
						+ 2\lfloor\log_2(d+1+d \lfloor\log_2(\lfloor r \rfloor)\rfloor)\rfloor +5
    \\
    N(\Psi ^{\lfloor r \rfloor}_f ) &= \max\left((1+\binom{d+\lfloor r \rfloor}{d} ) M^d\max(4, 2d+1)+2,
			2{\binom{d+\lfloor r \rfloor}{d}}(d+1+  d \lfloor\log_2(\lfloor r \rfloor)\rfloor)\right)
\end{align*}
such that
\begin{align*}
	&|\Psi_f^{\lfloor r\rfloor}(x)- f(x)|< \epsilon.
\end{align*}
Moreover, for any $x\in [-1,1)^d$, we have $\left|\Psi^{\lfloor r \rfloor}_f (x)\right|  \leq Re^{2d} $.
\end{lemma}
\begin{proof}
	The details can be found in Appendix \ref{ap:le5}.
\end{proof}

In the sequel, we construct a partition of unity using bump functions to approximate the target function $f$.
First, we let $\mathcal{P}_{2}$ be the partition defined in \eqref{partition}.
The bump function $w_{\P_2}$ is defined for  any $x\in\Rd$ and  $M\in \N$ as follows:
\begin{equation}\label{w_vb}
\begin{gathered}
	w_{\P_2}(x) = \prod_{k=1}^d \left(2\rho_2\left(\frac{M^2}{2}
		(-x_k +(\bold{C}_{\P_2}({x}))^\bL_k)+2\right)
		-4\rho_2\left(\frac{M^2}{2}(-x_k+(\bold{C}_{\P_2}({x}))^\bL_k)+\nicefrac 32\right)\right.
		\\
		\qquad\qquad
		\left.+4\rho_2\left(\frac{M^2}{2}(-x_k +(\bold{C}_{\P_2}({x}))^\bL_k) +\nicefrac 12\right)
		-2\rho_2\left(\frac{M^2}{2}(-x_k +(\bold{C}_{\P_2}({x}))^\bL_k)\right)\right).
\end{gathered}
\end{equation}
The function $w_{\P_2}$ reaches  its maximum (which is $1$) at the center of $\bold{C}_{\P_2}({x})$, decreases to zero near the boundary, and it vanishes at the boundary.
It is clear that $w_{\P_2}$ is the products of $d$ ReQU neural networks, each
with only one hidden layer containing $4$ neurons.
Therefore, using Lemma \ref{nle1}, we obtain the following result:
\begin{lemma}\label{le8}
	Let $ r>0$,  $\mathcal{P}_{2}$ be the partition defined in \eqref{partition}, 
	$M \in \N$ such that $M>>1$.
	Then for any
	$x \in \bigcup_{k \in \{1, \dots, M^{2d}\}}\ocirc{(C_{k})}_{1/M^{2r+2}}$
	 there exists a ReQU neural network
	$\Phi_{w_{\P_2}}\in \mathtt{N}_{\rho_2}(\lceil \log_2(d)\rceil +6, \max(4d, 2+ M^d\max(4, 2d+1)))$
	that  represents $w_{\P_2}(x)$, defined in \eqref{w_vb}, without error.
\end{lemma}
\begin{proof}
	The details can be found in Appendix \ref{ap:le8}.
\end{proof}

The previously constructed networks $\Psi_f^{\lfloor r\rfloor}$ of  Lemma \ref{le5}
to approximate  a given function $f\in \mathcal{H}^{r, R}$,
and $\Phi_{w_{\P_2}}$ of   Lemma \ref{le8} to represent  a bump function ${w_{\P_2}}$ are restricted
to $\bigcup_{k \in \{1, \dots, M^{2d}\}}\ocirc{(C_{k})}_{1/M^{2r+2}}$.
Therefore, we need to construct another network to control the approximation error
when $x$ belongs to   
$\bigcup_{k \in \{1, \dots, M^{2d}\}}C_{k}\setminus\ocirc{(C_{k})}_{1/M^{2r+2}}$.

\begin{lemma}\label{le9}
Let $\mathcal{P}_{1}$ and $\mathcal{P}_{2}$
be the partitions defined in \eqref{partition} and let $M \in \N$ such that $M>>1$.
Then there exists a ReQU neural network
$
\varphi_{\exists, \mathcal{P}_{2}}(x) \in \mathtt{N}_{\rho_2}(7, M^d(2d+1)+2d + 2dM^d)
$
satisfying
$$
\varphi_{\exists, \mathcal{P}_{2}}(x) =  \mathds{1}_{  \bigcup_{k \in \{1, \dots, M^{2d}\}}
	    C_{k} \setminus \ocirc{(C_{k})}_{1/M^{2r+2}}}
\;\text{where } x \notin \bigcup_{k \in \{1, \dots, M^{2d}\}}\ocirc{(C_{k})}_{1/M^{2r+2}} 
										\setminus\ocirc{ (C_{k})}_{2/M^{2r+2}}
$$
and that
$$
\varphi_{\exists, \mathcal{P}_{2}}(x) \in [0,1], \qquad
\text{where } x \in [-1,1)^d.
$$
\end{lemma}

\begin{proof}
	The details can be found in Appendix \ref{ap:le9}.
\end{proof}

\par

The next step is to approximate  the product $w_{\P_2}(x)f(x)$,
for any $x\in [-1,1)^d$.

\begin{theorem}\label{le10}
Let $r, R>0$,  $f\in \mathcal{H}^{r, R}(\mathbb{R}^d)$ and $M \in \N$  such that 
$M>\left(\frac{cRd^{\nicefrac r2}}{\epsilon}\right)^{\nicefrac{1}{2r}}$,
 for any $\epsilon \in (0, 1)$ and $c>0$ in \eqref{eq:taylor_error} .
Then there exists a ReQU neural network $ \Psi_f\in \mathtt{N}_{\rho_2}(L(\Psi_f), N(\Psi_f))$,
where 
\begin{align*}
L(\Psi_{f} ) &=
		\lfloor\log_2(\lfloor r \rfloor)\rfloor+ 2\lfloor\log_2(d+1+d \lfloor\log_2(\lfloor r \rfloor)\rfloor)\rfloor +8,
\\
N(\Psi_{f} ) &= \max\left((1+\binom{d+\lfloor r \rfloor}{d} ) M^d\max(4, 2d+1)+2,
			\;2{\binom{d+\lfloor r \rfloor}{d}}(d+1+  d \lfloor\log_2(\lfloor r \rfloor)\rfloor)\right) 
	\\
		&\qquad\qquad+ 2(M^d(2d+1)+2d + 2dM^d)+ \max(4d, 2+ M^d\max(4, 2d+1)).
\end{align*}
such that,
\begin{align*}
&\left|\Psi_f(x) - w_{\P_2}(x) \cdot f(x)\right|
		\leq \epsilon,
\end{align*}
for any $x \in [-1,1)^d$, where $w_{\P_2}$ defined in \eqref{w_vb}.
\end{theorem}

\begin{proof}
In the proof we use the ReQU neural networks  $\Psi^{\lfloor r \rfloor}_f $
and $\varphi_{\exists, \mathcal{P}_{2}}$ constructed in 
 Lemma \ref{le5} and   Lemma \ref{le9}, respectively.
First, we parallelize these networks and 
since the number of hidden layers in the construction of $\varphi_{\exists, \mathcal{P}_{2}}$
is less than the number of hidden layers in the construction of $\Psi^{\lfloor r \rfloor}_f $
we synchronize this by applying the identity ReQU network without explicitly write it.
Then, it is clear that the ReQU network
$$
\Phi^{\lfloor r \rfloor}_{f, \exists} (x)=\frac {1}{4Re^{2d}}
	\left(
	\rho_2(\Psi^{\lfloor r \rfloor}_f(x) - Re^{2d} \cdot \varphi_{\exists, \mathcal{P}_{2}}(x) + Re^{2d})
	+ \rho_2(-\Psi^{\lfloor r \rfloor}_f(x)- Re^{2d}\cdot \varphi_{\exists, \mathcal{P}_{2}}(x)+Re^{2d})
		\right)
$$
belongs to
$\mathtt{N}_{\rho_2}( L(\Phi^{\lfloor r \rfloor}_{f, \exists} ), N(\Phi^{\lfloor r \rfloor}_{f, \exists} ))$
where 
\begin{align*}
L(\Phi^{\lfloor r \rfloor}_{f, \exists} ) &=
		\lfloor\log_2(\lfloor r \rfloor)\rfloor+ 2\lfloor\log_2(d+1+d \lfloor\log_2(\lfloor r \rfloor)\rfloor)\rfloor +6,\\
N(\Phi^{\lfloor r \rfloor}_{f, \exists} ) &= \max\left((1+\binom{d+\lfloor r \rfloor}{d} ) M^d\max(4, 2d+1)+2,
			\;2{\binom{d+\lfloor r \rfloor}{d}}(d+1+  d \lfloor\log_2(\lfloor r \rfloor)\rfloor)\right) 
	\\
		&\qquad\qquad+ M^d(2d+1)+2d + 2dM^d.
\end{align*}
Similarly, we synchronize the number of hidden layers of $\varphi_{\exists, \mathcal{P}_{2}}$
and $\Phi^{\lfloor r \rfloor}_{f, \exists} (x)$
without explicitly write it. Consequently, we set
\begin{equation}\label{eq:Psi_f_E_def}
\Psi^{\lfloor r \rfloor}_{f, \exists} (x) =\phi_\times\left( 1-\varphi_{\exists, \mathcal{P}_{2}}(x),
		\Phi^{\lfloor r \rfloor}_{f, \exists} (x)\right),
\end{equation}
such that
 $\Psi^{\lfloor r \rfloor}_{f, \exists}$ belongs to
$\mathtt{N}_{\rho_2}( L(\Psi^{\lfloor r \rfloor}_{f, \exists} ),
    N(\Psi^{\lfloor r \rfloor}_{f, \exists} ))$,
where 
\begin{align*}
L(\Psi^{\lfloor r \rfloor}_{f, \exists} ) &=
		\lfloor\log_2(\lfloor r \rfloor)\rfloor+ 2\lfloor\log_2(d+1+d \lfloor\log_2(\lfloor r \rfloor)\rfloor)\rfloor +7,
	\\
N(\Psi^{\lfloor r \rfloor}_{f, \exists} ) &= \max\left((1+\binom{d+\lfloor r \rfloor}{d} ) M^d\max(4, 2d+1)+2,
			\;2{\binom{d+\lfloor r \rfloor}{d}}(d+1+  d \lfloor\log_2(\lfloor r \rfloor)\rfloor)\right) 
	\\
		&\qquad\qquad+ 2(M^d(2d+1)+2d + 2dM^d).
\end{align*}
Since $|\Psi^{\lfloor r \rfloor}_{f}(x)|\leq Re^{2d}$ , \cf   Lemma \ref{le5} ,
and that $ \varphi_{\exists, \mathcal{P}_{2}}(x) = 1$, for any
$x\in \bigcup_{k \in \{1, \dots, M^{2d}\}}     C_{k} \setminus \ocirc{(C_k)}_{1/M^{2r+2}}$,
it follows that $\Psi^{\lfloor r \rfloor}_{f, \exists} (x) =0$ when
$x\in \bigcup_{k \in \{1, \dots, M^{2d}\}}     C_{k} \setminus \ocirc{(C_k)}_{1/M^{2r+2}}$.
Let 	$\Phi_{w_{\P_2}}\in \mathtt{N}_{\rho_2}(\lceil \log_2(d)\rceil +6, 
				\max(4d, 2+ M^d\max(4, 2d+1)))$
be the network from  Lemma \ref{le8}, hence in order to multiply 
the networks $\Phi_{w_{\P_2}}$  and $\Psi^{\lfloor r \rfloor}_{f, \exists}$,
we need to parallelize them first, then we apply the ReQU network
$\phi_\times\in \mathtt{N}_{\rho_2}(1, 4)$, \cf  Lemma \ref{le2}.
To that aim, we synchronize their number of hidden layers
by successively applying the  identity ReQU network without explicitly write it.
In view of the characteristics of the used ReQU networks, it follows that
\begin{align*}
\Psi_f(x) =
	\phi_\times \left(\Phi_{w_{\P_2}}(x),
		\Psi^{\lfloor r \rfloor}_{f, \exists} (x)\right) \in \mathtt{N}_{\rho_2}(L(\Psi_f), N(\Psi_f)),
\end{align*}
where
\begin{align*}
L(\Psi_{f} ) &=
		\lfloor\log_2(\lfloor r \rfloor)\rfloor+ 2\lfloor\log_2(d+1+d \lfloor\log_2(\lfloor r \rfloor)\rfloor)\rfloor +8,\\
N(\Psi_{f} ) &= \max\left((1+\binom{d+\lfloor r \rfloor}{d} ) M^d\max(4, 2d+1)+2,
			\;2{\binom{d+\lfloor r \rfloor}{d}}(d+1+  d \lfloor\log_2(\lfloor r \rfloor)\rfloor)\right) 
	\\
		&\qquad\qquad+ 2(M^d(2d+1)+2d + 2dM^d)+ \max(4d, 2+ M^d\max(4, 2d+1))
	\\
		& = \max\left((1+\binom{d+\lfloor r \rfloor}{d} ) M^d\max(4, 2d+1)+2,
			\;2{\binom{d+\lfloor r \rfloor}{d}}(d+1+  d \lfloor\log_2(\lfloor r \rfloor)\rfloor)\right) 
	\\
		&\qquad\qquad+ 2(M^d(2d+1)+2d + 2dM^d)+ 2+ M^d\max(4, 2d+1).
\end{align*}

\par

In case  that
$x \in \bigcup_{k \in \{1, \dots, M^{2d}\}} \ocirc{(C_{k})}_{2/M^{2r+2}}$,
it is clear that 
\begin{equation}\label{eq:x_existence_case2}
x\notin \bigcup_{k \in \{1, \dots, M^{2d}\}} C_{k} 
		\setminus \ocirc{(C_{k})}_{1/M^{2r+2}}
\text{ and  } 
x\notin \bigcup_{i \in \{1, \dots, M^{2d}\}} 
		\ocirc{(C_{k})}_{1/M^{2r+2}} \setminus \ocirc{(C_{k})}_{2/M^{2r+2}}.
\end{equation}
Hence, in view of  Lemma \ref{le8}, $\Phi_{w_{\P_2}}$ represent ${w_{\P_2}}$, which is
defined in \eqref{w_vb}, without error.
Moreover, let $M>\left(\frac{cRd^{\nicefrac r2}}{\epsilon}\right)^{\nicefrac{1}{2r}}$,
for any $\epsilon \in (0, 1)$ and $c>0$  defined in \eqref{eq:taylor_error}.
Then, according to  Lemma \ref{le5}, the ReQU network  $\Psi_f^{\lfloor r\rfloor}(x)$
approximates $f$ up to an $\epsilon$ error.

In view of  \eqref{eq:x_existence_case2}
and   Lemma \ref{le9}, it follows that $\varphi_{\exists, \mathcal{P}_{2}} (x)=0$, 
together with \eqref{eq:Psi_f_E_def} imply that
$$
\Psi^{\lfloor r \rfloor}_{f, \exists} (x) =\frac 1{4Re^{2d}}\left(
	\rho_2\left(\Psi^{\lfloor r \rfloor}_f(x) +Re^{2d}\right)
	+ \rho_2\left(-\Psi^{\lfloor r \rfloor}_f(x)+Re^{2d}\right)\right).
$$
Moreover, using \eqref{eq:id_net_bounded_interval}, and the fact that
$\left| \Psi^{\lfloor r \rfloor}_f(x)\right|\leq Re^{2d}$ for any $x\in [-1,1)^d$ \cf   Lemma \ref{le5},
we get
$$
\Psi^{\lfloor r \rfloor}_{f, \exists} (x) = \Psi^{\lfloor r \rfloor}_f(x).
$$
Furthermore, using the fact that the maximum value attained by $w_{\P_2}$ is $1$,
 the approximation error of $\phi_\times$, $\Phi_{w_{\P_2}}$ and $\Psi^{\lfloor r \rfloor}_f$
in approximating the product, $w_{\P_2}$ and $f$ respectively,
it follows that
\begin{multline*}
\left|
\phi_\times \left(\Phi_{w_{\P_2}}(x), \Psi^{\lfloor r \rfloor}_{f, \exists} (x)\right)
- w_{\P_2}({x}) \cdot f({x})
\right|\\
 \leq \left|
\phi_\times \left(\Phi_{w_{\P_2}}(x), \Psi^{\lfloor r \rfloor}_{f, \exists} (x)\right)
					- \Phi_{w_{\P_2}}(x)\cdot \Psi^{\lfloor r \rfloor}_{f} (x)
\right|
  + \left|
\Phi_{w_{\P_2}}(x)\cdot \Psi^{\lfloor r \rfloor}_{f} (x)
				 - w_{\mathcal{P}_{2}}(x)\cdot \Psi^{\lfloor r \rfloor}_{f} (x)
\right|\\
    +
  \left|
  w_{\mathcal{P}_{2}}(x)\cdot \Psi^{\lfloor r \rfloor}_{f} (x) - w_{\P_2}({x}) \cdot f({x})
  \right|
 \leq \epsilon.
\end{multline*}

In case that $x\notin \bigcup_{k \in \{1, \dots, M^{2d}\}} C_{k} 
		\setminus \ocirc{(C_{k})}_{1/M^{2r+2}}$,
we have $\varphi_{\exists, \mathcal{P}_{2}}(x)=1$, which implies,
in view of \eqref{eq:Psi_f_E_def}, that
$\Psi^{\lfloor r \rfloor}_{f, \exists} (x) =0$. Furthermore, using the characterization in 
\eqref{Aj}  and \eqref{eq:A_charc_C}, we get
\begin{align*}
w_{\P_2}({x}) \leq  \frac{1}{  2M^{4r}} \leq \frac 12 (\frac{\epsilon}{cRd^{\nicefrac  r2}})^{2},
\end{align*}
hence we have
$$
\left|
\phi_\times \left(\Phi_{w_{\P_2}}(x), \Psi^{\lfloor r \rfloor}_{f, \exists} (x)\right)
- w_{\P_2}({x}) \cdot f({x})
\right|
 \leq \left|w_{\P_2}({x}) \cdot f({x}) \right|\leq  \frac 12 (\frac{\epsilon}{cRd^{\nicefrac  r2}})^{2}\cdot R 
\leq \epsilon.
$$

In case that $x\in \bigcup_{i \in \{1, \dots, M^{2d}\}} 
		\ocirc{(C_{k})}_{1/M^{2r+2}} \setminus \ocirc{(C_{k})}_{2/M^{2r+2}}$ but 
$x\notin \bigcup_{k \in \{1, \dots, M^{2d}\}} C_{k} 
		\setminus \ocirc{(C_{k})}_{1/M^{2r+2}}$,
 $\Psi^{\lfloor r \rfloor}_{f}(x)$ approximates $f(x)$ with an $\epsilon$ error.
Furthermore,  $\Phi_{w_{\P_2}}(x)$ approximates $w_{\P_2}(\bold{x})$ with no error,
such that $|\Phi_{w_{\P_2}}(x)| \leq |\Phi_{w_{\P_2}}(x) - w_{\P_2}({x})| + |w_{\P_2}({x})| \leq 1$.
Since $x\in \bigcup_{i \in \{1, \dots, M^{2d}\}} 
		\ocirc{(C_{k})}_{1/M^{2r+2}} \setminus \ocirc{(C_{k})}_{2/M^{2r+2}}$ 
and  $ \varphi_{\exists, \mathcal{P}_{2}}(x) \in [0,1]$, we have
$\left|{\Psi^{\lfloor r \rfloor}_{f, \exists} (x)}\right|\leq \left|\Psi^{\lfloor r \rfloor}_{f} (x)\right|\leq R + \epsilon$,
\cf   Lemma \ref{le5}.
Moreover, using the fact that 
\begin{align*}
w_{\P_2}({x}) \leq  \frac 12 (\frac{\epsilon}{cRd^{\nicefrac  r2}})^{2},
\end{align*}
we get

\begin{multline*}
	\left|
	\phi_\times \left(\Phi_{w_{\P_2}}(x),
	\Psi^{\lfloor r \rfloor}_{f, \exists} (x)\right) - w_{\P_2}({x}) \cdot f({x})
	\right|
	\\
	\leq
	\left|
	\phi_\times \left(\Phi_{w_{\P_2}}(x),
	\Psi^{\lfloor r \rfloor}_{f, \exists} (x)\right)
		- \Phi_{w_{\P_2}}(x)\cdot \Psi^{\lfloor r \rfloor}_{f, \exists} (x)
	\right|
	+ \left|
	\Phi_{w_{\P_2}}(x)\cdot \Psi^{\lfloor r \rfloor}_{f, \exists} (x)
		 - w_{\P_2}({x}) \cdot   \Psi^{\lfloor r \rfloor}_{f, \exists} (x)
	\right|
	\\
	 \quad +
	\left|
	w_{\P_2}({x}) \cdot   \Psi^{\lfloor r \rfloor}_{f, \exists} (x)
			-w_{\P_2}({x}) \cdot \Psi^{\lfloor r \rfloor}_{f} (x)
	\right|
	+
	\left|
	w_{\P_2}({x}) \cdot   \Psi^{\lfloor r \rfloor}_{f} (x)
			-w_{\P_2}({x}) \cdot f({x})
	\right|
	\\
	\leq \frac 12 (\frac{\epsilon}{cRd^{\nicefrac  r2}})^{2} (2R+3\epsilon )\leq \epsilon.
\end{multline*}
\end{proof}

In order to capture all the inputs from the cube $[-1,1]^d$,
we use a finite sum of those networks of  Theorem \ref{le10} constructed to
$2^d$ slightly shifted versions of $\P_2$.
Hence, we can approximate $f(x)$ on  $[-1,1]^d$.

\begin{proof}[Proof of Theorem \ref{thm_main}]
The approximation result in  Theorem \ref{le10} is independent on
the edges of the domain $[-1,1)^d$ and can be easily  extended to any symmetric bounded domain
of the form $[-a, a)^d$ where $a>0$.
Consequently, we restrict the proof to the 
cube $[-\nicefrac 12,\nicefrac 12]^d$ to show that 
there exist a ReQU network $\Phi_f$ satisfies
\begin{align*}
\sup_{x \in [-\nicefrac 12,\nicefrac 12]^d} \left|\Phi_f(x) - f(x)\right| \leq \epsilon.
\end{align*}
We denote by $\mathcal{P}_{1,\kappa}$ and $\mathcal{P}_{2,\kappa}$,
the modifications of $\mathcal{P}_1 := \mathcal{P}_{1,1}$ and $\mathcal{P}_2:= \mathcal{P}_{2,1}$ ,
respectively, defined in  \eqref{partition}, such that
at least one of the  components is shifted by $\nicefrac 1{M^2}$ for $\kappa \in \{2, 3, \dots, 2^d\}$.
Moreover, we denote by $C_{k,\kappa}$ the corresponding cubes of the partition $\mathcal{P}_{2,\kappa}$
 such that $k \in \{1, \dots, M^{2d}\}$ and $\kappa \in \{1, \dots, 2^d\}$.
In case $d=2$, we have $2^2$ partitions,  as the following figure shows:
\begin{figure}[h!]
\centering
\begin{minipage}{0.20\textwidth}
\begin{tikzpicture}
\draw[step=0.5cm,color=gray, thick] (-1,-1) grid (1,1);
\end{tikzpicture}
\end{minipage}
\begin{minipage}{0.22\textwidth}
\begin{tikzpicture}
\draw[step=0.5cm,color=gray, thick] (-1,-1) grid (1,1);
\draw[step=0.5cm,color=gray, thick, xshift=0.25cm, dashed] (-1,-1) grid (1,1);
\end{tikzpicture}
\end{minipage}
\begin{minipage}[c]{0.20\textwidth}
\begin{tikzpicture}
\draw[step=0.5cm,color=gray, thick, yshift=0.5cm] (-1,-1) grid (1,1);
\draw[step=0.5cm,color=gray, thick, yshift=0.75cm, dashed] (-1,-1) grid (1,1);
\end{tikzpicture}
\end{minipage}
\begin{minipage}{0.20\textwidth}
\begin{tikzpicture}
\draw[step=0.5cm,color=gray, thick, yshift=0.25cm] (-1,-1) grid (1,1);
\draw[step=0.5cm,color=gray, thick, yshift=0.5cm, xshift=0.25cm, dashed] (-1,-1) grid (1,1);
\end{tikzpicture}
\end{minipage}
\caption{$2^2$ partitions in two dimensions.}
\label{fig8}
\end{figure}

 Figure \ref{fig8} shows that if we shift our partition along at least one component by the same additional distance,
we get $2^2=4$ different partitions that include all the data in the domain.
The main idea is to compute a linear combination of ReQU neural networks from   Lemma \ref{le5} 
$\Psi_{f, \mathcal{P}_{2,\kappa}}^{\lfloor r\rfloor}$ for the partitions $\mathcal{P}_{2,\kappa}$
where $\kappa \in \{1, \dots, 2^d\} $, respectively.
Note that $\Psi_{f, \mathcal{P}_{2,1}}^{\lfloor r\rfloor}:= \Psi_f^{\lfloor r\rfloor}$
given  in  Lemma \ref{le5}.
Moreover, we use the bump function defined in \eqref{w_vb}as a weight to avoid
approximation error increases near to the boundary of any cube
of the partitions. Hence we multiply $\Psi_{f, \mathcal{P}_{2,\kappa}}^{\lfloor r\rfloor}$  by the following weight function
\begin{equation}\label{w_v}
\begin{gathered}
	w_{\P_{2, \kappa}}(x) = \prod_{k=1}^d \left(2\rho_2\left(\frac{M^2}{2}(-x_k +(\bold{C}_{\P_{2, \kappa}}({x}))^\bL_k)+2\right)
		-4\rho_2\left(\frac{M^2}{2}(-x_k+(\bold{C}_{\P_{2, \kappa}}({x}))^\bL_k)+\nicefrac 32\right)\right.
		\\
		\qquad\qquad
		\left.+4\rho_2\left(\frac{M^2}{2}(-x_k +(\bold{C}_{\P_{2, \kappa}}({x}))^\bL_k) +\nicefrac 12\right)
		-2\rho_2\left(\frac{M^2}{2}(-x_k +(\bold{C}_{\P_{2, \kappa}}({x}))^\bL_k)\right)\right).
\end{gathered}
\end{equation}
As a bump function $w_{\P_{2, \kappa}}(x)$ is supported in $C_{\P_{2,\kappa}}({x})$,
and attains its maximum at the center of $C_{\P_{2,\kappa}}({x})$.
Moreover, $\{w_{\P_{2, \kappa}}(x)\}_\kappa$ is a partition of unity
for any $x \in [-\nicefrac 12,\nicefrac 12]^d$,
that is $w_{P_{2,1}}({x})+ \dots + w_{P_{2,2^d}}({x}) = 1$, for any $x \in [-\nicefrac 12,\nicefrac 12]^d$.

Let $\Psi_{f,\kappa}$ be the ReQU networks of  Theorem \ref{le10}
corresponding to the partitions $\mathcal{P}_{2,\kappa}$ where $\kappa \in \{1, \dots, 2^d\}$,
respectively.
Moreover, the fact that  $[-1/2, 1/2]^d \subset [-1+1/M^2, 1)^d$ implies that
each of $\P_{1,\kappa}$ and $\P_{2,\kappa}$
form a partition  which contains $[-1/2,1/2]^d$
and the approximation error in  Theorem \ref{le10} holds for each ReQU network
$\Psi_{f,\kappa}$ on $[-1/2,1/2]^d$.
Furthermore, the final ReQU network $\Phi_f(x) $ belongs to
$ \mathtt{N}_{\rho_2}(L(\Phi_f), N(\Phi_f))$
and constructed as follow
$$
\Phi_f(x)  =  \sum_{v=1}^{2^d} \Psi_{f,\kappa}(x)
$$
where 
\begin{align*}
L(\Phi_f ) &=
		\lfloor\log_2(\lfloor r \rfloor)\rfloor+ 2\lfloor\log_2(d+1+d \lfloor\log_2(\lfloor r \rfloor)\rfloor)\rfloor +8,\\
N(\Phi_f) &= 2^d\Big(\max\left((1+\binom{d+\lfloor r \rfloor}{d} ) M^d\max(4, 2d+1)+2,
			\;2{\binom{d+\lfloor r \rfloor}{d}}(d+1+  d \lfloor\log_2(\lfloor r \rfloor)\rfloor)\right)
	\\
		&\qquad\qquad+ 2(M^d(2d+1)+2d + 2dM^d)+ 2+ M^d\max(4, 2d+1)\Big).
\end{align*}
Using the properties of $\{w_{\P_{2, \kappa}}\}_\kappa$, we have
\begin{align*}
f({x}) = \sum_{\kappa=1}^{2^d} w_{\P_{2, \kappa}}({x}) \cdot f({x}).
\end{align*}
Using  Theorem \ref{le10} and the notations form its proof, for 
the networks $\Phi_{w_{\P_{2, \kappa}}}$ and $\Psi^{\lfloor r \rfloor}_{f, \exists, \P_{2, \kappa}}$
for the partitions $\P_{2, \kappa}$ where $\kappa \in \{ 1, \dots, 2^d\}$ respectively, such that 
$\Phi_{w_{\P_{2, 1}}}:=\Phi_{w_{\P_{2}}}$ and
$\Psi^{\lfloor r \rfloor}_{f, \exists, \P_{2,1}} := \Psi^{\lfloor r \rfloor}_{f, \exists}$.
Consequently, for $\epsilon = \epsilon'/2^d $ such that $\epsilon'\in (0, 1)$, we get
\begin{align*}
  \left|\Phi_f(x)  ({x}) - f({x})\right| &= \left|\sum_{\kappa=1}^{2^d}
		\phi_\times \left(\Phi_{w_{\P_{2, \kappa}}}(x), \Psi^{\lfloor r \rfloor}_{f, \exists, \P_{2, \kappa}} (x)\right)
		- \sum_{v=1}^{2^d} w_{\P_{2, \kappa}}({x})\cdot f(\bold{x})\right|\\
& \leq \sum_{\kappa=1}^{2^d} \left|
	\phi_\times \left(\Phi_{w_{\P_{2, \kappa}}}(x), \Psi^{\lfloor r \rfloor}_{f, \exists, \P_{2, \kappa}} (x)\right)
		 -  w_{\P_{2, \kappa}}({x}) \cdot f(\bold{x})\right| \leq \epsilon'.
\end{align*}
\end{proof}

\section*{Acknowledgments}

The author would like to thank Dr. Thomas Dittrich for his helpful discussions regarding this paper.



\printbibliography

\newpage

\appendix

\section{Proofs for Section \ref{se3}}
\subsection{Proof of Lemma \ref{le5}}\label{ap:le5}

In order to prove Lemma \ref{le5} we need some preliminary results.
We show that ReQU neural networks can represent the identity function within a bounded symmetric domain.
In particular, we show that using the ReQU activation function, it is possible to construct a shallow ReQU neural network (with only two neurons in the hidden layer) that accurately represents the mapping $f(x)=x$ for any $x\in [-1, 1]^d$.
\begin{align}
	{\phi}_{id}(t) &= \frac{\rho_2(t+1) - \rho_2(-t+1)}{4} = t, \quad t \in [-1, 1] \label{eq:id_requ}
	\intertext{and, for $x \in [-1, 1]^d$, we have}
	{\Phi}_{id}({x}) &= \left(\phi_{id}\left(x_1\right), \dots, \phi_{id}\left(x_d\right)\right)
	= \left(x_{1}, \dots, x_{d}\right)=x.\nonumber
\end{align}
The network $\Phi_{id}$ can be used to synchronize the number of hidden layers between two networks.
Additionally, it can be employed to adjust the input values for the succeeding layer.
For this purpose, we introduce the following notations:
\begin{align}
	\begin{split}\label{eq:id_net_n}
		\Phi_{id}^0({x}) &= {x}, \quad {x} \in [-1, 1]^d\\
		\Phi_{id}^{n+1}({x}) &= \Phi_{id}\left(\Phi_{id}^n({x})\right) = {x},
		\quad n \in \N_0, {x} \in [-1,1]^d.
	\end{split}
\end{align}

It is evident that we can extend \eqref{eq:id_requ} to any symmetric bounded interval. Specifically, for $s>0$ then 
\begin{equation}\label{eq:id_net_bounded_interval}
	\phi_{id, s} (t)= \frac{1}{4s} \left( \rho_2(t+s) - \rho_2(-t+s) \right) = t \text{ for any } t \in [-s, s].
\end{equation}

In the following result, we demonstrate that a ReQU network can represent the product of two inputs.
Mainly, we can construct a shallow neural network with the ReQU activation function
that represents the product operator using one hidden layer with 4 neurons.
The proof of the following result is straightforward and thus left as an exercise for the reader.

\begin{lemma}\label{le2}
	Let $\rho_2: \mathbb{R} \to \R$ be the ReQU activation function, and
	$$
	\phi_\times(x, y) = 1/4\left(\rho_2(x+y) +\rho_2(-x-y) -\rho_2(-x+y) -\rho_2(x-y)\right).
	$$
	Then, for any $x, y\in \mathbb{R}$, the ReQU network $\phi_\times (x, y)$ represents
	the product $xy$ without error.
\end{lemma}

Moreover,  in the next lemma we show that ReQU neural networks are capable of representing the  product of the input vector $x\in \Rd$.

\begin{lemma}\label{nle1}
	For any $x\in \Rd$,
	there exists a ReQU neural network
	$\Phi_{\tinyprod, d}\in \mathtt{N}_{\rho_2}(\lceil \log_2(d)\rceil, 4d)$
	that can represent the product $\prod_{k=1}^dx_k$
	without error.
\end{lemma}

\begin{proof}
	In order to construct the network $\Phi_{\tinyprod, d}$,
	first we append the input data, that is, 
	\begin{equation*}\label{neq1}
		(z_1, \dots, z_{2^q})=
		\left(x_1, \dots, x_d, \underbrace{1, \dots,1}_{2^q-d} \right).
	\end{equation*}
	where $q=\lceil \log_2(d)\rceil$.
	Next, we use the ReQU neural network $\phi_\times$ form Lemma \ref{le2},
	which can  represent
	the product of two inputs $x$ and $y$ for any $x, y\in \R$,
	using a single hidden layer with 4 neurons.
	In the first hidden layer of $\Phi_{\tinyprod, d}$,
	we compute 
	\[ 
	\phi_\times(z_1,z_2),  \phi_\times(z_3,z_4),  \dots, \phi_\times(z_{2^q-1},z_{2^q}), 
	\]
	resulting in a vector with $2^{q-1}$  entries.
	We then pair these outputs and apply $\phi_\times$ again.
	This procedure continues until only one output remains.
	Therefore, we need $q$ hidden layers, each containing at most $4d$ neurons.
\end{proof}

\par

We define $\mathrm{P_N}$ as the  linear span of all monomials of the form $\prod_{i =1}^d x_i^{r_i}$,
where $r_1, \dots, r_d \in \N_0$,  such that $r_1+\dots+r_d \leq N$.
Hence, $\mathrm{P_N}$ is a linear vector space of functions of dimension
$\frac {(N+d)!}{N!d!}$, indeed
\begin{align*}
	dim \ \mathrm{P_N} =
	\left|\left\{(r_1, \dots, r_d) \in \N_0^{d}: r_1+\dots+r_d \leq N \right\}\right|
	= \binom{d+N}{d}.
\end{align*}

\par

\begin{lemma}\label{le3}
	Let $s>0$, $m_1, \dots, m_{\binom{d+N}{d}}$ denote all monomials in $\mathrm{P_N}$
	for some $N \in \N$.
	Let $w_1, \dots, w_{\binom{d+N}{d}} \in \R$, define
	\begin{equation}\label{eq:polynom}
		p\left({x}, y_1, \dots, y_{\binom{d+N}{d}}\right) =
		\sum_{i=1}^{\binom{d+N}{d}} w_i \cdot y_i \cdot m_i({x}),
		\quad  {x} \in [-s,s]^d, y_i \in [-s,s].
	\end{equation}
	Then there exists  a ReQU neural network
	$
	\Phi_{p}\left( {x}, y_1, \dots, y_{\binom{d+N}{d}}\right)
	$
	with at most 
	$\lfloor\log_2(N)\rfloor + 2\lfloor\log_2(d+1+d \lfloor\log_2(N)\rfloor)\rfloor +1$
	hidden layers 
	and
	$2{\binom{d+N}{d}}(d+1+  d \lfloor\log_2(N)\rfloor)$
	neurons in each hidden layer,
	such that
	\begin{align*}
		\Phi_{p}\left( {x}, y_1, \dots, y_{\binom{d+N}{d}}\right) =
		p\left( {x}, y_1, \dots, y_{\binom{d+N}{d}}\right)  
	\end{align*}
	for all $ {x} \in [-s,s]^d$, $y_1, \dots, y_{\binom{d+N}{d}} \in [-s,s]$.
\end{lemma}

\begin{proof}
	Basically we construct a neural network $\phi_m$,
	which is able to represent a monomial
	$y \cdot m(x)$ where $m \in \mathrm{P_N}$, $ x \in [-s, s]^d$ and $y \in [-s, s]$,
	of the following form:
	\begin{equation*}
		y \cdot m(x) = y\cdot \prod_{k=1}^d (x_k)^{r_k},\text{ such that }
		r_1, \dots, r_d \in \N_0,\text{  and }r_1+ \dots + r_d \leq N.
	\end{equation*}
	
	Since $r_k	\in \mathbb{N}_0$ where $k\in \{1, \dots, d\}$,
	and given that any number can be represented  as sums of power of $2$, where the highest power corresponds to the maximum number of compositions required
	by ReQU activation functions in order to represent $(x_k)^{r_k}$.
	We assume without loss of generality that $r_k\neq 0$ and let  
	$$
	Z = ( y, \underbrace{x_1, x_1^2, x_1^4, \dots, x_1^{\iota_1}}_{\lfloor\log_2(r_1)+1\rfloor},
	\dots, \underbrace{x_d, x_d^2, \dots, x_d^{\iota_d}}_{\lfloor\log_2(r_d)+1\rfloor})
	$$
	where $\iota_k = 2^{\lfloor\log_2(r_k)\rfloor}$, $k \in \{1, \dots, d\}$.
	Each block of variables $I_k := \{x_k, x_k^2, \dots, x_k^{\iota_k}\}$ generates $(x_k)^{r_k}$ using the appropriately chosen factors in $I_k$
	where cardinality of $I_k$ is $\lfloor\log_2(r_k)+1\rfloor$ such that $k \in \{1, \dots, d\}$. 
	Therefore, to produce $(x_k)^{r_k}$
	we need at most $\lfloor\log_2(r_k)+1\rfloor$ elements of $I_k$.
	To generate all elements of $Z$, we need
	$1 + \lfloor\log_2(r_1)+1\rfloor + \dots +\lfloor\log_2(r_d)+1\rfloor
	=d+ 1+ \sum_{k=1}^{d} \lfloor\log_2(r_k)\rfloor$ networks.
	For example to produce $x_1, x_1^2, x_1^4, \dots, x_1^{\iota_1}$
	we need the identity network, \cf \eqref{eq:id_net_bounded_interval}, for $x_1$
	and the shallow network for $x_1^2$,
	a network with 2 hidden layers each contain one neuron to present $x_1^4$
	and so on.

	We then parallelize all these networks.
	To ensure the same number of hidden layers, we concatenate them with the identity network given in \eqref{eq:id_net_n}.
	Mainly,  the  $d+1+\sum_{k=1}^{d} \lfloor\log_2(r_k)\rfloor$ networks have
	at most $\max_{k \in \{1, \dots, d\}}\lfloor\log_2(r_k)\rfloor$ hidden layers, with each layer containing at most $2$ neurons. Consequently, the parallelized network has
	$\max_{k \in \{1, \dots, d\}}\lfloor\log_2(r_k)\rfloor$ hidden layers, with at most $2(d+ 1+\sum_{k=1}^{d} \lfloor\log_2(r_k)\rfloor)$ neurons per layer.
	
	In this step, all outputs of our constructed network are encoded in the vector $Z$.
	The products of all the elements in $Z$ produce the monomial $ym(x)$, hence
	we need at most $2\lfloor\log_2(|Z|)\rfloor$ additional hidden layers to realize the final product.
	Each of these $2\lfloor\log_2(|Z|)\rfloor$  layers has at most $4\lceil\frac{|Z|}2\rceil$ neurons,
	such that $|Z| = d+1+  \sum_{k=1}^{d} \lfloor\log_2(r_k)\rfloor$.
	Consequently, $ym(x)$ is represented exactly by $\phi_m$ which is a neural network
	with $d+1$ input dimension, $1$ output dimension, and 
	it has at most $\max_{k \in \{1, \dots, d\}}\lfloor\log_2(r_k)\rfloor + 2\lfloor\log_2(|Z|)\rfloor$
	hidden layers, each  containing at most
	$2(d+1+  \sum_{k=1}^{d} \lfloor\log_2(r_k)\rfloor)$
	neurons.
	
	Considering this construction and the fact that
	$p(x, y_1, \dots, y_{\binom{d+N}{d}}) $ is given by \eqref{eq:polynom},
	there exist neural networks $\phi_{m_i}$ such that
	$$
	p(x, y_1, \dots, y_{\binom{d+N}{d}}) =
	\sum_{i=1}^{\binom{d+N}{d}} r_i \phi_{m_i}(x, y_i).
	$$
	Hence the final network has at most 
	$\max_{k \in \{1, \dots, d\}}\lfloor\log_2(r_k)\rfloor + 
	2\lfloor\log_2(d+1+  \sum_{k=1}^{d} \lfloor\log_2(r_k)\rfloor)\rfloor +1$
	hidden layers, each with at most
	$2{\binom{d+N}{d}}(d+1+  \sum_{k=1}^{d} \lfloor\log_2(r_k)\rfloor)$
	neurons.
	Since $r_k$ can be bounded by $N$, this completes the proof.
\end{proof}

\begin{remark}
	Let $\rho_p(x)= \max(0, x)^p$ for any  $p \in \mathbb{N}$, and  $x\in \mathbb{R}$
	and let $y, s>0$, we have
	\begin{align*}
		\rho_p\left(y -s\cdot \rho_p\left(x\right)\right) =\begin{cases}
			0\ &\mbox{for} \ x \geq (\frac{y}{s})^{\frac 1p}\\
			y^p \ &\mbox{for} \ x \leq 0.
		\end{cases}
	\end{align*}
	Moreover,  for $y\in \mathbb{R}$, it is true that 
	\begin{align*}
		\rho_p\left(\phi_{id}(y)-s\cdot \rho_p\left(x\right)\right)
		+\rho_p\left(-\phi_{id}(y)-s\cdot \rho_p\left(x\right)\right) =
		\begin{cases}
			0 \ &\mbox{for} \ x \geq (\frac{|y|}{s})^{\frac 1p}\\
			|y|^p \ &\mbox{for} \ x \leq 0.
		\end{cases}
	\end{align*}
	In particular if $|y|\leq s$, we get
	\begin{align*}
		\rho_p\left(\phi_{id}(y)-s\cdot \rho_p\left(x\right)\right)
		-\rho_p\left(-\phi_{id}(y)-s\cdot \rho_p\left(x\right)\right) =
		\begin{cases}
			0 \ &\mbox{for} \ x \geq (\frac{|y|}{s})^{\frac 1p}\\
			sign(y)|y|^p \ &\mbox{for} \ x \leq 0.
		\end{cases}
	\end{align*}
\end{remark}

\par

Since it is not hard to show the correctness of the previous statements,
we leave the details for the reader.

\par

\begin{lemma}\label{le4}
	Let $s>0$,  $a, b\in \Rd$, such that
	$b_i - a_i \geq \frac{2}{s} \ \text{ for all} \ i \in \{1, \dots, d\}$
	and let
	\begin{align*}
		&U_{s} = \big\{x\in \Rd: x_i \notin [a_i, a_i+\nicefrac 1s) \cup (b_i - \nicefrac 1s, b_i),
		\text{ for all} \ i \in \{1, \dots, d\}\big\}.
	\end{align*}
	\begin{enumerate}
		\item Then there exists a ReQU neural network with two hidden layers, $2d$ neurons in the first layer and
		one neuron in the second layer denoted by $\Phi_{\mathds{1}_{[a,b)}}$ such that
		\begin{align*}
			\Phi_{\mathds{1}_{[a,b)}}(x) &= 
			\rho_2\bigg(1- s^2 \cdot \sum_{i=1}^d \left(\rho_2\left(- x_i + a_i + \nicefrac{1}{s} \right)
			+ \rho_2\left(x_i - b_i + \nicefrac{1}{s}\right) \right)\bigg)
		\end{align*}
		satisfies,  for any $x \in U_s$,
		$\Phi_{\mathds{1}_{[a,b)}}(x) = \mathds{1}_{ [{a}, {b})}({x})$
		and 
		$
		\left|\Phi_{\mathds{1}_{[a,b)}}(x) - \mathds{1}_{[a, b)}(x)\right| \leq 1 \text{ for } x \in \Rd.
		$
		
		\item Let $y\in \mathbb{R}$ such that $|y|\leq s$.
		Then there exists a ReQU neural network $\Phi_{\times,\mathds{1}}$,
		with $d+1$ input dimension, $3$ hidden layers,
		$2d + 1$ neurons in the first layer, two neurons in the second hidden layer and
		$4$ neurons in the last hidden layer, satisfies
		\begin{align*}
			&\Phi_{\times,\mathds{1}}( x, y; a, b) = \phi_{\times}\left(\phi_{id, s}(\phi_{id, s}(y)),
			\Phi_{\mathds{1}_{[a,b)}}(x) \right) 
			= y \cdot \mathds{1}_{ [{a}, {b})}({x}),\quad \text{for any } x \in U_{s} 
			\intertext{and}
			&\left|\Phi_{\times,\mathds{1}}( x, y; a, b) -
			y \cdot \mathds{1}_{ [{a}, {b})}({x})\right| \leq |y|, \quad \text{where } x\in \Rd.
		\end{align*}
	\end{enumerate}
\end{lemma}

\par

\begin{proof}
	The proof of the first result  in the lemma can be concluded in a similar way as in 
	the proof of a) in \cite[Lemma 6]{Kohler21rateconvergencefully}.
	The second result in our lemma is straightforward. Indeed,
	using the identity neural network defined in \eqref{eq:id_net_bounded_interval}, to
	update the number of hidden layers in the representation of $y$.
	Then we use the product ReQU neural network given in Lemma \ref{le2}
	to multiply the output of $\phi_{id, s}(\phi_{id, s}(y))$ and $\Phi_{\mathds{1}_{[a,b)}}(x) $.
	Here the network $\left(\phi_{id, s}(y)), \Phi_{\mathds{1}_{[a,b)}}(x)\right) $
	is a  ReQU neural network with $d+1$ input dimension and
	$2$ output dimension, two hidden layers such that in the fist layer we have $2d+1$
	neurons and the second has only two neurons .
	Then the number of hidden layers in the composition 
	$\phi_{\times}\left(\phi_{id, s}(\phi_{id, s}(y)), \Phi_{\mathds{1}_{[a,b)}}(x) \right) $
	equals to $3$ hidden layers.
	The first hidden layer  has $2d+1$ neurons, the second has two neurons and the third has 4 neurons.
\end{proof}

\begin{proof}[Proof of  Lemma \ref{le5}]
	We begin by demonstrating that ReQU neural networks can approximate
	the recursively constructed  function $\psi ^{\lfloor r \rfloor}_f$ from Lemma \ref{supple3}.
	Let $s \in \N$ and $y=(y_1,\dots, y_d)\in \Rd$ such that $|y_i| \leq s$  for each  $i\in \{1, \dots , d\}$.
	By  Lemma \ref{le4}, for any  $a, b\in \Rd$, where
	$b_i - a_i \geq \frac{2}{s}$ for all $i \in \{1, \dots, d\}$,
	and for
	$x\in \Rd$ such that $ x_i \notin [a_i, a_i+\nicefrac 1s) \cup (b_i - \nicefrac 1s, b_i),
	\text{ for all} \ i \in \{1, \dots, d\}$,
	we have:
	\begin{align*}
		\Phi_{\mathds{1}_{[a,b)}}(x) &= \mathds{1}_{ [{a}, {b})}({x})
		\intertext{and}
		\Phi_{\times,\mathds{1}}( x, y; a, b) 
		&= \left(\Phi_{\times,\mathds{1}}( x, y_1,; a, b),\dots,
		\Phi_{\times,\mathds{1}}( x, y_d; a, b)\right)
		= y \cdot \mathds{1}_{ [{a}, {b})}({x}).
	\end{align*}
	Next, using   Lemma \ref{le3},  we establish the existence of a ReQU neural network
	$\Phi_{p}\left( {x}, z_1, \dots, z_{\binom{d+\lfloor r \rfloor}{d}}\right)$
	with at most 
	$\lfloor\log_2(\lfloor r \rfloor)\rfloor + 2\lfloor\log_2(d+1+d \lfloor\log_2(\lfloor r \rfloor)\rfloor)\rfloor +1$
	hidden layers 	and
	$2{\binom{d+\lfloor r \rfloor}{d}}(d+1+  d \lfloor\log_2(\lfloor r \rfloor)\rfloor)$
	neurons per hidden layer,
	such that
	\begin{align*}
		\Phi_{p}\left( {z}, \zeta_1, \dots, \zeta_{\binom{d+\lfloor r \rfloor}{d}}\right) =
		p\left( {z}, \zeta_1, \dots,\zeta_{\binom{d+\lfloor r \rfloor}{d}}\right)  
	\end{align*}
	where $z_1, \dots,z_d, \zeta_1, \dots, \zeta_{\binom{d+\lfloor r \rfloor}{d}}\in [-\tau, \tau]$, with
	$\tau = \max\left\{2, R\right\}$
	\cf  \eqref{fp} and \eqref{eq:max_bound}.
	
	We proceed by representing the recursion in \eqref{eq:rec_phi} and \eqref{eq:rec_psi}
	by the appropriate ReQU neural networks:
	\begin{flalign}\label{eq:net_rec_phi}
		&&	\Phi^{(0)} &= (\Phi^{(0)} _{ 1}, \dots, \Phi^{(0)} _{d}) = \Phi_{id}(x),&\notag
		\\[1ex]
		&&	\Phi^{(1)}& = (\Phi^{(1)}_{ 1}, \dots, \Phi^{(1)}_{ d}) =
		\sum_{j \in \{1, \dots, M^d\}} (\bold{B}_{j})^{\bold{L}} \cdot \Phi_{\mathds{1}_{B_j}}(x)&\notag
		\\[1ex]
		\rlap{and}\\
		&&	\Phi^{(\alpha, i)} _f&=
		\sum_{j \in \{1, \dots, M^d\}} (\partial^{\alpha} f)
		\left((\bold{{C}}_{i, j})^{\bL}\right) \cdot\Phi_{\mathds{1}_{B_j}}(x),&\notag
	\end{flalign}
	for $i \in \{1, \dots, M^d\}$ and  $\alpha \in \N_0^d$ such that $| \alpha | \leq \lfloor r \rfloor$.
	Similarly, we present the recursion for $\psi_f^{(\alpha)}$ from  \eqref{eq:psi_def} through the following networks:
	\begin{flalign}\label{eq:net_rec_psi}
		&&\Psi^{(0)} &= (\Psi^{(0)}_{1}, \dots, \Psi^{(0)}_{d})= \Phi_{id} ({\Phi}^{(0)}),&\notag
		\\[1ex]
		&&\Psi^{(1)}&= (\Psi^{(1)}_{1}, \dots, \Psi^{(1)}_{d})&\notag
		\\[1ex]
		\rlap{such that}\notag\\
		&& \Psi^{(1)}_k &=
		\sum_{i=1}^{M^d} \Phi_{\times,\mathds{1}}( \Phi^{(0)},\Phi_k^{(1)}+v_k^{(i)}; 
		\Phi^{(1)}+v^{(i)}, \Phi^{(1)}+v^{(i)}+\nicefrac{2}{M^2}\cdot 1_{\Rd}),
		\quad k\in \{ 1, \dots, d\} &\notag
		\\
		\rlap{and}\\
		&&\Psi^{(\alpha)}_f &= \sum_{i=1}^{M^d} \Phi_{\times,\mathds{1}}( \Phi^{(0)},\Phi^{(\alpha, i)}_f; 
		\Phi^{(1)}+v^{(i)}, \Phi^{(1)}+v^{(i)}+\nicefrac{2}{M^2}\cdot 1_{\Rd})&\notag
	\end{flalign}
	for $i \in \{1, \dots, M^d\}$ and  $\alpha \in \N_0^d$ such that $| \alpha | \leq \lfloor r \rfloor$.
	Let  $\alpha \in \mathbb{N}_0^{\binom{d+\lfloor r \rfloor}{d}}$,
	such that $\|  \alpha\|_{\ell^0}= d$ and $|\alpha| \leq \lfloor r \rfloor$.
	Using $\Phi_p$  from Lemma \ref{le3}, we represent $\psi ^{\lfloor r \rfloor}_f$ in \eqref{eq:psi_def}
	by the following ReQU neural network:
	\begin{flalign}\label{fp}
		&&\Psi ^{\lfloor r \rfloor}_f(x)&=
		\Phi_p\left(z, \zeta_1, \dots, \zeta_{\binom{d+{\lfloor r \rfloor}}{d}}\right),&\notag
		\\
		\rlap{where}
		\\
		&&z&= \Psi^{(0)}- \Psi^{(1)}
		\quad \text{and}\quad  
		\zeta_k = \Psi^{(\alpha_k)}_f &\notag
	\end{flalign}
	for $k \in \left\{1, \dots, \binom{d+{\lfloor r \rfloor}}{d}\right\}$.
	The coefficients $w_1, \dots, w_{\binom{d+{\lfloor r \rfloor}}{d}}$ in  Lemma \ref{le3} are chosen as
	\begin{align}\label{eq:wight_val}
		w_k=
		\frac{1}{\alpha_k!}, \quad k \in \left\{1, \dots, \binom{d+{\lfloor r \rfloor}}{d}\right\}.
	\end{align}
	The  neural networks $\Phi^{(0)}, \Phi^{(1)}, \Phi^{(\alpha, i)} _f$ such that $i \in\{1,\dots, M^d\}$
	and $\Psi^{(0)} , \Psi^{(1)}, \Psi^{(\alpha_k)}_f $ where $k\in\{ 1, \dots, \binom{d+{\lfloor r \rfloor}}{d}\}$
	are computed in parallel.
	Hence, the total number of layers in the final constructed network
	equals the maximum number of layers among the parallel networks.
	The ReQU realization of the network architecture $\Phi^{(0)}$ requires only one hidden layer with two neurons, whereas the construction of  $\Psi^{(0)}$ involves
	two  hidden layers, each with two neurons.
	
	Since  $\Phi^{(1)} = \sum_{j \in \{1, \dots, M^d\}}
	(\bold{B}_{j})^{\bold{L}} \cdot \Phi_{\mathds{1}_{B_j}}(x)$,
	we need $3$ hidden layers with $2d + 1 $ neurons in the first layer,
	two neurons in the second and $4$ neurons in the last hidden layer to achieve 
	$\Phi_{\mathds{1}_{B_j}}(x)$ for each fixed $j$.
	That is, we need $M^{d}$ times the complexity of  
	$\Phi_{\mathds{1}_{B_j}}(x)$ in order to get $\Phi^{(1)}$.
	Consequently,  $\Phi^{(1)}\in \mathtt{N}_{\rho_2}(3 , M^d\max(4, 2d+1))$.
	For any $i \in\{1,\dots, M^d\}$, we have
	$$
	\Phi^{(\alpha, i)} _f\in \mathtt{N}_{\rho_2}(3 , M^d\max(4, 2d+1)),
	$$
	since it is constructed similarly to $\Phi^{(1)}$.
	Therefore,  the network
	$$
	\left(\Phi^{(0)},\Phi^{(1)}, \Phi^{(\alpha, 1)} _f, \dots, \Phi^{(\alpha, M^d)} _f \right)
	\in \mathtt{N}_{\rho_2}\left(3, (1+M^d ) M^d\max(4, 2d+1)+2 \right).
	$$
	Similarly, we conclude that 
	$$
	\left(\Psi^{(0)},\Psi^{(1)}, \Psi^{(\alpha_1)}_f, \dots, \Psi^{(\alpha_{\binom{d+\lfloor r \rfloor}{d}})} _f \right)
	\in \mathtt{N}_{\rho_2}\left(5, (1+\binom{d+\lfloor r \rfloor}{d} ) M^d\max(4, 2d+1)+2 \right).
	$$
	Finally,  since $\Psi ^{\lfloor r \rfloor}_f(x)$ in \eqref{fp}  is the composition of $\Phi_p$
	and
	$$
	( \Psi^{(0)},\Psi^{(1)}, \Psi^{(\alpha_1)}_f, \dots, \Psi^{(\alpha_{\binom{d+\lfloor r \rfloor}{d}})} _f ),
	$$
	we conclude that
	$$
	\Psi ^{\lfloor r \rfloor}_f (x)\in
	\mathtt{N}_{\rho_2}\left(L(\Psi ^{\lfloor r \rfloor}_f ),N(\Psi ^{\lfloor r \rfloor}_f ) \right),
	$$ 
	where
	\begin{align*}
		L(\Psi ^{\lfloor r \rfloor}_f ) &=\lfloor\log_2(\lfloor r \rfloor)\rfloor
		+ 2\lfloor\log_2(d+1+d \lfloor\log_2(\lfloor r \rfloor)\rfloor)\rfloor +5
		\\
		N(\Psi ^{\lfloor r \rfloor}_f ) &= \max\left((1+\binom{d+\lfloor r \rfloor}{d} ) M^d\max(4, 2d+1)+2,
		2{\binom{d+\lfloor r \rfloor}{d}}(d+1+  d \lfloor\log_2(\lfloor r \rfloor)\rfloor)\right).
	\end{align*}
	
	\par
	
	It remains to determine the approximation error of the network $\Psi ^{\lfloor r \rfloor}_f (x)$,
	for  $s\geq 1/M^{2r+2}$  and any $x \in \bigcup_{k \in \{1, \dots, M^{2d}\}}\ocirc{(C_{k})}_{1/M^{2r+2}}$.
	Thanks to Lemma \ref{le4} the ReQU neural networks
	$$
	\Phi^{(0)},\Phi^{(1)}, \Phi^{(\alpha, 1)} _f, \dots, \Phi^{(\alpha, M^d)} _f,
	\Psi^{(0)},\Psi^{(1)}, \Psi^{(\alpha)}_f,  \text{ where }\alpha\in\mathbb{N}_0^{\binom{d+\lfloor r \rfloor}{d}}
	$$ 
	represent the following functions respectively without error
	$$
	\phi^{(0)},\phi^{(1)}, \phi^{(\alpha, 1)} _f, \dots, \phi^{(\alpha, M^d)} _f,
	\psi^{(0)},\psi^{(1)}, \psi^{(\alpha)}_f,  \text{ where }\alpha\in\mathbb{N}_0^{\binom{d+\lfloor r \rfloor}{d}}.
	$$ 
	Using the previous conclusions and the recursion provided in \eqref{eq:rec_phi} and \eqref{eq:rec_psi},
	we have 
	\begin{flalign}\label{eq:max_bound}
		&&	&\left|\Psi^{(0)}-\Psi^{(1)}\right| = 
		\left|x- \sum_{i=1}^{M^d}
		({\phi}^{(1)}+{v}^{(i)}) \cdot \mathds{1}_{\mathcal{A}^{(i)}} 
		\left({\phi}^{(0)}\right) \right| \leq 2&\notag\\
		\rlap{and}\\
		&& 	&\left| \Psi^{(\alpha_k)}_f\right| = 
		\left|\psi^{(\alpha_k)}_f\right| \leq \|f\|_{C^{\lfloor r \rfloor}([-1,1]^d}\leq R, 
		\text{ where } k \in \left\{1, \dots, \binom{d+{\lfloor r \rfloor}}{d}\right\}.&\notag
	\end{flalign}
	The last inequality follows from the fact that $f$ belongs to $ \mathcal{H}^{r, R}(\Rd)$.
	In view of the previous construction and  Lemma \ref{le3},  we have
	\begin{equation}\label{eq:net_taylor_error}
		|\Psi ^{\lfloor r \rfloor}_f (x)- T_{(\bold{C}_{\mathcal{P}_2}({x}))^{\bL}}^{\lfloor r \rfloor}f({x})|
		= |\Psi ^{\lfloor r \rfloor}_f (x)-\psi ^{\lfloor r \rfloor}_f|= 0.
	\end{equation}
	Consequently, using \eqref{eq:taylor_error}, \eqref{eq:net_taylor_error},
	and the fact that for any $\epsilon \in (0, 1)$,
	$M>\left(\frac{cRd^{\nicefrac r2}}{\epsilon}\right)^{\nicefrac{1}{2r}}$
	and $x \in \bigcup_{k \in \{1, \dots, M^{2d}\}}\ocirc{(C_{k})}_{1/M^{2r+2}}$,
	we get
	$$
	|\Psi^{\lfloor r \rfloor}_f (x) - f (x)|\leq
	|\Psi^{\lfloor r \rfloor}_f (x) - T_{(\bold{C}_{\mathcal{P}_2}({x}))^{\bL}}^{\lfloor r \rfloor}f({x}) |
	+|T_{(\bold{C}_{\mathcal{P}_2}({x}))^{\bL}}^{\lfloor r \rfloor}f({x})  - f(x) | <\epsilon,
	$$
	which gives the first result in the lemma.
	Next, using the previous inequality and the fact that $f\in \mathcal{H}^{r, R}$,
	we show the bound on the constructed network $\Psi^{\lfloor r \rfloor}_f(x)$.
	Therefore, for any $x \in \bigcup_{k \in \{1, \dots, M^{2d}\}}\ocirc{(C_{k})}_{1/M^{2r+2}}$
	
	\begin{align*}
		\left|\Psi^{\lfloor r \rfloor}_f (x)\right| \leq &
		|\Psi^{\lfloor r \rfloor}_f (x) - T_{(\bold{C}_{\mathcal{P}_2}({x}))^{\bL}}^{\lfloor r \rfloor}f({x}) |
		+|T_{(\bold{C}_{\mathcal{P}_2}({x}))^{\bL}}^{\lfloor r \rfloor}f({x})  - f(x) | 
		+ \left|f(x)\right|
		\\
		\leq & \epsilon + \sup_{x \in [-1,1]^d} \left|f(x)\right| \leq 2\max(\epsilon, R).
	\end{align*}
	It remains to show an upper bound for the network $\Psi^{\lfloor r \rfloor}_f (x)$
	when $x$ belongs to  $\bigcup_{k \in \{1, \dots, M^{2d}\}}C_{k}\setminus \ocirc{(C_{k})}_{1/M^{2r+2}}$.
	Since  in this case the networks $\Phi_{\mathds{1} _{B_k}}$ and $\Phi_{\times,\mathds{1}}$
	are not exact, for any $x\in B_k$ such that $k\in \{ 1, \dots, M^d\}$, we obtain the following:
	
	\begin{align*}
		\left|{\Phi}_f^{(\alpha, k)}\right|& \leq 
		\left|(\partial^{\alpha} f)  \left((\bold{{C}}_{i, j})^{\bL}\right)\right|,
		\quad \text {for any }k \in \{1, \dots, M^d\}
		\intertext{ and }
		\left|\Phi_j^{(1)}\right|& \leq 1  \quad \text{ where } \in j \in  \{1, \dots, d\}.
	\end{align*}
	In view of construction of $\Psi_f^{(\alpha)}$ \cf \eqref{eq:net_rec_psi},
	and the fact that there exists
	at most a non zero element in the sum in \eqref{eq:net_rec_psi} for $\Psi_f^{(\alpha)}$,
	it follows that 
	\begin{align*}
		\left|\Psi_f^{(\alpha)} \right| \leq \|f\|_{C^{\lfloor r\rfloor}([-1,1]^d)}
		\intertext{ and }
		\left|\Psi_j^{(2)}\right| \leq 1, \quad \text{ where } j \in \{1, \dots, d\}.
	\end{align*}
	In conclusion, using \eqref{fp},  \eqref{eq:polynom}, \eqref{eq:wight_val},  we get 
	\begin{eqnarray*}
		\left|\Psi^{\lfloor r \rfloor}_f (x)\right|  &\leq &\left|\Phi_{p}
		\left( {x}, y_1, \dots, y_{\binom{d+N}{d}}\right) -
		p\left( {x}, y_1, \dots, y_{\binom{d+N}{d}}\right)  \right|
		+ \left| p\left( {x}, y_1, \dots, y_{\binom{d+N}{d}}\right) \right|\\
		&\leq& \left| p\left( {x}, y_1, \dots, y_{\binom{d+N}{d}}\right) \right|
		\leq \sum_{0 \leq |\alpha| \leq \lfloor r\rfloor }
		\frac{1}{\alpha!} \cdot \|f\|_{C^{\lfloor r\rfloor }([-1,1]^d)} \cdot 2^{|\alpha|}\\
		&\leq&R\cdot\left(\sum_{l=0}^\infty \frac{(2a)^l}{l!}\right)^d \leq Re^{2d}.
	\end{eqnarray*}
	Since, $2\max(\epsilon, R)< Re^{2d}$, we conclude the result of the lemma.
\end{proof}
\subsection{Proof of Lemma \ref{le8}}\label{ap:le8}
The proof of the lemma follows from Lemma \ref{nle1}, and the fact that
for any $k\in \{ 1, \dots, d\}$,
\begin{equation}\label{w_pk}
	\begin{gathered}
		\left(2\rho_2\left(\frac{M^2}{2}(-x_k +(\bold{C}_{\P_2}({x}))^\bL_k)+2\right)
		-4\rho_2\left(\frac{M^2}{2}(-x_k+(\bold{C}_{\P_2}({x}))^\bL_k)+\nicefrac 32\right)\right.
		\\
		\left.+4\rho_2\left(\frac{M^2}{2}(-x_k +(\bold{C}_{\P_2}({x}))^\bL_k) +\nicefrac 12\right)
		-2\rho_2\left(\frac{M^2}{2}(-x_k +(\bold{C}_{\P_2}({x}))^\bL_k)\right)\right)
	\end{gathered}
\end{equation}
is a ReQU neural network with one hidden layer containing $4$ neurons.
To determine the value of $\bold{C}_{\P_2}({x}))^\bL$, we use the construction of
$\Phi^{(1)}$ and  $\Psi^{(1)}$ given in the proof of Lemma \ref{le5}. Thus,
we need a ReQU neural network with $5$ hidden layers each containing at most $M^d\max(4, 2d+1)$ neurons to compute the value of $\bold{C}_{\P_2}({x}))^\bL$.
We use the identity network to update the number of hidden layers for the input $x$,
hence $x$ and $\bold{C}_{\P_2}({x}))^\bL$ can be represented as two parallel networks  with $2d$ outputs,
with $5$ hidden layers  each  with at most $2+ M^d\max(4, 2d+1)$ neurons.
The computation of \eqref{w_pk} requires two inputs from the later parallelized networks, hence
to get \eqref{w_pk}  for all $k\in \{1, \dots, d\}$,
we need a ReQU network with  one hidden layer contains $4d$ neurons.
Therefore, using Lemma \ref{nle1}, the final constructed network
$\Phi_{w_{\P_2}}\in \mathtt{N}_{\rho_2}(\lceil \log_2(d)\rceil +6, \max(4d, 2+ M^d\max(4, 2d+1)))$
represents $w_{\P_2}(x) $ in  \eqref{w_vb} without error.

\subsection{Proof of Lemma \ref{le9}}\label{ap:le9}
	In view of \eqref{partition},
	for any $k\in \{ 1, \dots, M^d\}$, we denote $C_{\mathcal{P}_1}({x}) = B_k$.
	As a first step, the network will  check whether a given input $x$ exists in
	$\bigcup_{k \in \{1, \dots, M^d\}} B_{k}\setminus (\ocirc{B_k})_{1/M^{2r+2}}$.
	To achieve this, we construct the following function:
	\begin{align*}
		g_1({x}) &= \mathds{1}_{\bigcup_{k \in \{1, \dots, M^d\}}
			B_k\setminus (\ocirc{B_k})_{1/M^{2r+2}}}({x})
		=1-\sum_{k \in \{1, \dots, M^d\}} \mathds{1}_{(\ocirc{B_k})_{1/M^{2r+2}}}({x}).
	\end{align*}
	This function is approximated by the following ReQU neural network
	\begin{equation}\label{eq:varphi_1}
		\varphi_1(x)=
		1-\sum_{k \in \{1, \dots, M^d\}} \Phi_{\mathds{1}_{ (\ocirc{B_k})_{1/M^{2r+2}}}}(x),
	\end{equation}
	where $\Phi_{\mathds{1}_{ (\ocirc{B_k})_{1/M^{2r+2}}}}(x)$ for $k \in \{1, \dots, M^d\}$
	are the networks described in  Lemma \ref{le4}.
	The ReQU neural netwok $\varphi_1$ belongs to $\mathtt{N}_{\rho_2}(2, 2dM^d)$, as shown in   Lemma \ref{le4}.
	Using  $\Phi^{(1)}$ from \eqref{eq:net_rec_phi} which belongs to $\mathtt{N}_{\rho_2}(3, M^d(2d+1))$,
	we can determine the position of  $(\bold{B}_k)^{\bold{L}}$
	to approximate the indicator functions on $\P_2$ for the cubes $C_{k} \subset C_{\P_1}(x)$.
	To synchronize the number of hidden layers in the parallelized networks that construct
	$(\bold{B}_k)^{\bold{L}}$ and $x$, we need to apply the identity network, to $x$, 3 times.
	Hence,  
	$$
	x=  \Phi_{id }(\Phi_{id}(\Phi_{id }(x)))\in \mathtt{N}_{\rho_2}(3, 2d).
	$$
	Inspired by  \eqref{Aj}, we characterize the cubes
	$(\ocirc{C}_{ i, j})_{1/M^{2r+2}}$, $i \in \{1, \dots, M^d\}$, that are contained in the cube $B_{j}$,
	by
	\begin{equation}\label{eq:A_charc_C}
		\begin{gathered}
			\ocirc{(\mathcal{A}^{(i)})}_{1/M^{2r+2}} =
			\left\{x \in \Rd: -x_k + \phi^{(1)}_k + v_k^{(i)} +\frac{1}{M^{2r+2}}\leq 0 \right.
			\qquad\qquad\qquad\qquad\qquad
			\\
			\left. \qquad\qquad\qquad\qquad\qquad
			\mbox{and} \ x_k - \phi^{(1)}_k - v_k^{(i)}- \frac{2}{M^2} +\frac{1}{M^{2r+2}} < 0
			\ \mbox{for all} \ k \in \{1, \dots, d\}\right\}.
		\end{gathered}
	\end{equation}
	Therefore, the following function:
	\begin{align*}
		g_2(x)
		=\mathds{1}_{\bigcup_{i\in \{1, \dots, M^{d}\}} C_{i,j}
			\setminus \ocirc{(C_{i,j})}_{1/M^{2r+2}}}(x)
		=1-\sum_{i \in \{1, \dots, M^d\}} \mathds{1}_{\ocirc{({C}_{i, j})}_{1/M^{2r+2}}}(x)
	\end{align*}
	can be approximated by the ReQU neural network $\varphi_2$ defined as:
	\begin{align*}
		\varphi_2(x) &=
		1- \!\!\!\!\!\sum_{i \in \{1, \dots, M^d\}}\!\!\!\!\!\Phi_{\times,\mathds{1}}\left(\Phi_{id}^3(x),1;
		\Phi^{(1)}+v^{(i)} +\frac{1}{M^{2r+2}}\cdot 1_{\Rd}, 
		\Phi^{(1)} +v^{(i)} +(\frac{2}{M^2} -\frac{1}{M^{2r+2}})\cdot 1_{\Rd}\right),
	\end{align*}
	where $\Phi_{\times,\mathds{1}}$ is the network described in  Lemma \ref{le4}, which belongs to 
	$\mathtt{N}_{\rho_2}(3, 2d+1)$.
	Moreover, since  $\Phi^{(1)}\in \mathtt{N}_{\rho_2}(3, M^d(2d+1))$
	and $\Phi_{id}^3 \in \mathtt{N}_{\rho_2}(3, 2d)$,
	it follows that $\varphi_2\in \mathtt{N}_{\rho_2}(6, M^d(2d+1)+2d)$.
	\par
	Using the previous constructed ReQU neural networks $\varphi_1$ and $\varphi_2$,
	we define our final network $\varphi_{\exists, \mathcal{P}_{2}}$ as follows: 
	\begin{align*}
		\varphi_{\exists, \mathcal{P}_{2}}(x) &= 1-\rho_2\left(1-\varphi_2(x)
		- \phi_{id}^4\left(\varphi_1(x)\right)\right).
	\end{align*}
	It is clear that  $\varphi_{\exists, \mathcal{P}_{2}}(x)\in \{0, 1\}$; moreover, 
	it belongs to $\mathtt{N}_{\rho_2}(7, M^d(2d+1)+2d + 2dM^d)$.
	Next, if $x\notin \bigcup_{k \in \{1, \dots, M^{2d}\}}
	\ocirc{(C_{k})}_{1/M^{2r+2}}\setminus\ocirc{ (C_{k})}_{2/M^{2r+2}}$,
	we show that 
	\[
	\varphi_{\exists, \mathcal{P}_{2}}(x)= 
	\mathds{1}_{
		\bigcup_{k \in \{1, \dots, M^{2d}\}}
		C_{k} \setminus \ocirc{(C_{k})}_{1/M^{2r+2}}
	}(x).
	\]
	First, we  consider the case where:
	$$
	x\notin\bigcup_{k \in \{1, \dots, M^d\}} \ocirc{(B_{k})}_{1/M^{2r+2}}
	\text{ which implies that }
	x\notin \bigcup_{k \in \{1, \dots, M^{2d}\}}\ocirc{ (C_{k})}_{1/M^{2r+2}}.
	$$
	From the construction given  in \eqref{eq:varphi_1},  it is clear that 
	in this case that $\varphi_1(x)=1$.
	Consequently, $1- \varphi_2(x) -\phi_{id}^4(\varphi_1(x)) = - \varphi_2(x)$.
	Since $\varphi_2 \geq 0$, then
	
	$$
	\varphi_{\exists, \mathcal{P}_{2}}(x) =
	1-\rho_2\left(-\varphi_2(x)\right) = 1= \mathds{1}_{
		\bigcup_{k \in \{1, \dots, M^{2d}\}}
		C_{k} \setminus \ocirc{(C_{k})}_{1/M^{2r+2}}
	}(x).
	$$
	
	Next, we assume that $x$ belongs to the following intersection
	$$
	\bigcup_{k \in \{1, \dots, M^d\}} \ocirc{(B_{k})}_{1/M^{2r+2}}\,
	\bigcap
	\, \bigcup_{k \in \{1, \dots, M^{2d}\}}\ocirc{ (C_{k})}_{2/M^{2r+2}}.
	$$

	Since we only concerned by 
	$
	x \notin \bigcup_{k \in \{1, \dots, M^{2d}\}}\ocirc{(C_{k})}_{1/M^{2r+2}} 
	\setminus\ocirc{ (C_{k})}_{2/M^{2r+2}}
	$, in our statement, we conclude  that
	\begin{align*}
		\Phi_{\times,\mathds{1}}\left(\Phi_{id}^3(x),1;
		\Phi^{(1)}+v^{(i)} +\frac{1}{M^{2r+2}}\cdot 1_{\Rd}, 
		\Phi^{(1)} +v^{(i)} +(\frac{2}{M^2} -\frac{1}{M^{2r+2}})\cdot 1_{\Rd}\right)
		=
		\mathds{1}_{ \ocirc{(C_{i, j})}_{1/M^{2r+2}}}(x)
	\end{align*}
	for all $i \in \{1, \dots, M^d \}$,
	where we used the characterization of $\ocirc{(C_{i, j})}_{1/M^{2r+2}}$
	in \eqref{eq:A_charc_C} and   Lemma \ref{le4}.
	Which implies that $\varphi_2(x)= g_2(x)=0$. Moreover, since
	$x \in \bigcup_{k \in \{1, \dots, M^{2d}\}} \ocirc{(C_{k})}_{2/M^{2r+2}}$,
	it follows that $x \in \bigcup_{k \in \{1, \dots, M^{d}\}} \ocirc{(B_{k})}_{2/M^{2p+2}}$.
	Hence, in view of  Lemma \ref{le4}, we conclude that
	$\varphi_1(x) = g_1(x)=0$.
	Consequently $1-\varphi_2(x) - \phi_{id}^4\left(\varphi_1(x)\right) =1$,
	which implies that
	$$
	\varphi_{\exists, \mathcal{P}_{2}}(x) = 1-\rho_2\left(1-\varphi_2(x)
	- \phi_{id}^4\left(\varphi_1(x)\right)\right)= 0
	=  \mathds{1}_{\bigcup_{k \in \{1, \dots, M^{2d}\}}
		C_{k} \setminus \ocirc{(C_{k})}_{1/M^{2r+2}}}(x).
	$$
	
	Finally, we assume that $x$ belongs to the following domain
	$$
	\bigcup_{k \in \{1, \dots, M^d\}} \ocirc{(B_{k})}_{1/M^{2r+2}}\,
	\bigcap
	\bigcup_{k \in \{1, \dots, M^{2d}\}} {(C_{k})} \setminus \ocirc{(C_{k})}_{1/M^{2r+2}},
	$$
	which implies that 
	$x \notin \bigcup_{k \in \{1, \dots, M^{2d}\}} \ocirc{(C_{k})}_{1/M^{2r+2}}$.
	Considering  Lemma \ref{le4}, in this situation
	$ \varphi_1(x) \in [0,1]$.
	In a similar way to the previous case,
	since $x\in \bigcup_{k \in \{1, \dots, M^d\}} \ocirc{(B_{k})}_{1/M^{2r+2}}$,
	it follows that  $ \varphi_2(x) = g_2(x)=1$.
	To sum up,  we have
	$$
	1-\varphi_2(x) - \phi_{id}^4\left(\varphi_1(x)\right)  
	= \sum_{i \in \{1, \dots, M^d\}} \mathds{1}_{\ocirc{({C}_{i, j})}_{1/M^{2r+2}}}(x)
	- \phi_{id}^4\left(\varphi_1(x)\right)  
	\leq 0.
	$$
	Which implies that 
	$$
	\varphi_{\exists, \mathcal{P}_{2}}(x) = 1
	= \mathds{1}_{  \bigcup_{k \in \{1, \dots, M^{2d}\}}
		C_{k} \setminus \ocirc{(C_{k})}_{1/M^{2r+2}}
	}(x).
	$$
	To conclude, by all the previous constructions, it follows that
	$$
	\varphi_{\exists, \mathcal{P}_{2}}(x) =  \mathds{1}_{  \bigcup_{k \in \{1, \dots, M^{2d}\}}
		C_{k} \setminus \ocirc{(C_{k})}_{1/M^{2r+2}}}
	\;\text{where } x \notin \bigcup_{k \in \{1, \dots, M^{2d}\}}\ocirc{(C_{k})}_{1/M^{2r+2}} 
	\setminus\ocirc{ (C_{k})}_{2/M^{2r+2}}
	$$
	and that
	$$
	\varphi_{\exists, \mathcal{P}_{2}}(x) \in [0,1], \qquad
	\text{where } x \in [-1,1)^d.
	$$

\section{ReQU network approximation of the square root}

Next result is of  independent interest, where we show that
ReQU neural networks can approximate the square root.

\begin{lemma} \label{lem:sqrt_NN}
    For any $\epsilon \in (0,1)$, there exists a ReQU neural network $\phi_{\sqrt{}}$
    satisfies the following:
    \begin{equation*}
    	\sup_{x\in [0,t]}|\sqrt{x} - \phi_{\sqrt{} } (x)|
    	\leq 
    	\epsilon,
    \end{equation*}
    such that $\phi_{\sqrt{}}$ has at most $\mathcal{O}(n)$ layers,
    $\mathcal{O}(n^2)$ neurons and $\mathcal{O}(n^3)$  weights, where 
    $$
    	n\geq {\log\left( t(\log(1/2) +3 \log(\epsilon^{-1}))\epsilon^{-2}\right) }/{\log(2)}.
    $$
\end{lemma}

\begin{proof}
	The proof relies on an iterative method for the root extraction originally published
	in \cite{Gower58noteiterativemethod}, and extended to ReLU neural networks in \cite{Grohs20Deepneuralnetworka}.
	Hence, we use some similar idea for the ReQU neural networks.
\par
	The case where $x=0$ is not important, hence let $x\in (0,t]$  where
	$t\geq 1$. Then for every $n\in\mathbb{N}$ we define the sequences 
	\begin{equation}\label{eq:appendix_def_sn_cn}  
	s_{n+1} = s_n - \frac{s_n c_n}{2t^2}  
	\quad \text{and} \quad 
	c_{n+1} = c_n^2 \frac{c_n -3 t}{4t^2}
	\end{equation}
	with $s_0 = x/\sqrt{t}$ and $c_0 = x-t$.
	For every $n\in \mathbb{N}$,  we have 
	$t+c_{n+1} = (t+c_n)(1-\frac{c_n}{2t})^2$, 
	which implies by induction that for every $n\in\mathbb{N}_0$
	\begin{equation}\label{eq:appendix_eq_1}
	x(t+c_n) = ts_n^2.
	\end{equation}
	Since  $c_n$ is a decreasing sequence and $t\geq 1$,
	then $|(c_n -3t)/4t^2|\leq 1/t$. Therefore, by induction, for every $k, n\in\mathbb{N}_0$
	such that $k\leq n$ we have
	\begin{equation*}\label{eq:quad_conv_est}
	|c_n|
	\leq 
	\frac{|c_{n-k}|^{2^{k}}}{t^{2^k-1}},
	\end{equation*}
	which implies with~\eqref{eq:appendix_eq_1} 
	that for every $n\in\mathbb{N}_0$ 
	\begin{equation}\label{eq:appendix_err_est}
	|x-s_n^2|= \frac xt|c_n| 
	\leq|c_n|
	\leq \frac{|c_0|^{2^n}}{t^{2^n-1}}.
	\end{equation}
	Since $t$ is fixed and $|c_0|/t<1$, for any  $x\in (0,t]$, then $s_n \to \sqrt{x}$ as $n\to\infty $. 
	In order to guaranty the uniform convergence with respect to $x$,
	we rewrite the sequence in \eqref{eq:appendix_def_sn_cn}  with shifted initial data for every $x\in (0,t]$
	\begin{equation*}
	 s_0 = \frac{x+\epsilon^2}{\sqrt{t}}
	 \quad \text{and} \quad 
	 c_0  = x+\epsilon^2 -t,\quad \text{where }\; \epsilon\in (0,1).
	\end{equation*}

	By~\eqref{eq:appendix_err_est}, for every $n\in\mathbb{N}_0$
	\begin{align*}
		   |\sqrt{x+\epsilon^2} - s_n|
		   &\leq 
		   \frac{|x+\epsilon^2-s_n^2|}{\sqrt{x+\epsilon^2}+s_n}
		   \leq 
		   \frac{\sqrt{t}|x+\epsilon^2-s_n^2|}{2\sqrt{x+\epsilon^2}}
		   \\
		   &\leq 
		   \frac{|c_0|^{2^n}}{2t^{2^n-\nicefrac 32}\sqrt{x+\epsilon^2}}
		   =
		   \frac{(t-(x+\epsilon^2))^{2^n}}{2t^{2^n-\nicefrac 32}(x+\epsilon^2)}
		   \leq 
		   \frac{(t-\epsilon^2)^{2^n}}{2t^{2^n-\nicefrac 32}\epsilon^2}  .
	\end{align*}
	 The inequality $\frac{{(t-\epsilon^2)^{2^n}}}{{2t^{2^n-\nicefrac 32}\epsilon^2}}\leq \epsilon$
	 holds true  if $2^n \geq t(\log(1/2) + 3\log(\epsilon^{-1})) \epsilon^{-2}.$
	 Indeed,  if  $(t-\epsilon^2)^{2^n}\leq 2t^{2^n-\nicefrac 32}\epsilon^3\leq 2t^{2^n}\epsilon^3$
	 then  $(t-\epsilon^2)^{2^n}\leq 2t^{2^n}\epsilon^3$.
	 Therefore, $2^n\log(t/(t-\epsilon^2))\geq \log(1/2) +3\log(\epsilon^{-1})$,
	 using the fact that $\log(t/(t-\epsilon^2))= \log\frac{1}{1-\epsilon^2/t}\geq \epsilon^2/t$,
	 we get $2^n\geq t( \log(1/2) +3\log(\epsilon^{-1}))\epsilon^{-2} $.

   	In view of the fact that $|\sqrt{x} - \sqrt{x+\epsilon^2}|\leq \epsilon$ for any $x\in (0,t]$ and $t\geq 1$, 
	we have
	\begin{equation*}\label{eq:err_bound_s_n}
		\sup_{x\in (0,t]}  |\sqrt{x} - s_n |   \leq    \epsilon    \quad    \text{for}\quad 
		n \geq 
		\frac{\log\left( t(\log(1/2) +3 \log(\epsilon^{-1}))\epsilon^{-2}\right) }{\log(2)}.
   \end{equation*}

We construct  ReQU neural networks that realize the iteration in \eqref{eq:appendix_def_sn_cn} .
Let $\phi_\times$ be  the neural network from Lemma \ref{le2}, hence, for fixed $t\geq 1$,
we let
\begin{equation*}
	\overline{s}_{n+1} = \phi_\times\left(\overline{s}_n , 1- \frac{\overline{c}_n}{2t^2}  \right)
	\quad \text{and} \quad 
	\overline{c}_{n+1} = \phi_\times\left(\rho_2( \overline{c}_n)+
	\rho_2(- \overline{c}_n) , \frac{\overline{c}_n -3 t}{4t^2}\right)
\end{equation*}
with $\overline s_0 = s_0$ and $\overline{c}_0 =c_0$. 
The ReQU neural network $\phi_\times$ can represent the product
of any two real number without error with only one hidden layer which contains 4 neurons.
That is  $\overline{c}_n = c_n$ and $\overline{s}_n = s_n$, hence we have for any $\epsilon \in (0, 1)$
$$
\sup_{x\in (0,t]}  |\sqrt{x} - \overline{s}_n |   \leq    \epsilon    \quad    \text{for}\quad 
		n \geq 
		\frac{\log\left( t(\log(1/2) +3 \log(\epsilon^{-1}))\epsilon^{-2}\right) }{\log(2)}.
$$
It remains to determine the network complexity.
In order to get $\overline{c}_n$ we need a concatenation of the product network $\phi_\times$
and parallelized networks.
We prove by induction that the number of layers for $\overline{c}_n$ is $2n +1$.
If $n=0$, it is clear that to represent $\overline{c}_1$ we need only 3 layers.
The number of layers $L$ in the construction of $\overline{c}_{n+1}$ equals to
$(2n+1 + 2 -1)+2 -1= 2(n+1) +1$, in view of
\cite[Proposition 1]{Cheridito19Efficientapproximationhighdimensional}.

Since the maximum number of neurons per layer is 4  neurons in the product network and
2 neurons in each network to be parallelized, and in view the fact that
composition of  these networks  will only increase the number of neurons in the parallelized part
of the resulting network, we conclude that 4 neurons is the maximum number of neurons per layer
in the construction of $\overline{c}_n$. Consequently, $4(L-1) +1= 8n +1$
is the maximum number of neurons to construct $\overline{c}_n$.
Therefore 
$$\underbrace{4(L-1) +1}_{biases} + \underbrace{4^2*(L-2) +2*4}_{weights}= 40n -7$$
 is the maximum number of weights .
Moreover, in order to determine the number or layers in the construction of $\overline{s}_n$,
we note  that $\overline{c}_n$ has more layers than $\overline{s}_n$.
Hence the missed layers in the parallelization are added by the composition
the identity network.  Using \cite[Proposition 1]{Cheridito19Efficientapproximationhighdimensional} and the fact that
$\overline{c}_n$ needs $2n+1$ layers, it follows that  the construction of $\overline{s}_n$,
needs $2n+2$ layers. The parallelization phase combine two different networks
the first if for  $\overline{c}_n$ that contain 4 neurons in each layer, the second is $\overline{s}_n$,
which has  $4n$ neurons in each layer, this can be deduced by induction
where details are left to the reader.  The maximum number of weights in $\overline{s}_n$
are $2n(2n+1) +1 +(2n)^2(2n)  +2*2n$.
Finally, we conclude that in order to get the desired network $\phi_{\sqrt{}}=\overline{s}_n$
that approximate $\sqrt{x}$, we need at most $\mathcal{O}(n)$ layers,
$\mathcal{O}(n^2)$ neurons and $\mathcal{O}(n^3)$  weights, where 
$n\geq \frac{\log\left( t(\log(1/2) +3 \log(\epsilon^{-1}))\epsilon^{-2}\right) }{\log(2)}$.
\end{proof}

\end{document}